\documentclass[a4paper,12pt]{article}

\usepackage[left = 1.25in, right = 1.25in, top = 1.25in, bottom = 1.25in]{geometry}
\usepackage{amsmath,amsfonts,amssymb,amsthm,rotating,bbm} 
\usepackage[mathscr]{euscript}
\usepackage{bm,bbm}
\usepackage{mathtools} 
\usepackage[authoryear,round]{natbib}
\usepackage{xr}
\usepackage{xcolor}
\usepackage{setspace}
\usepackage{fncylab}
\usepackage{graphicx}
\usepackage{enumerate}
\usepackage{hyperref}
\usepackage[T1]{fontenc}
\usepackage[utf8]{inputenc}
\allowdisplaybreaks

\newtheorem{defn}{Definition}
\newtheorem{cond}{Condition}
\newtheorem{prop}{Proposition}
\newtheorem*{prop-nonumber}{Proposition}
\newtheorem{lemma_main}{Lemma}
\newtheorem{lemma}{Lemma}[section]
\newtheorem*{lemma-nonumber}{Lemma}
\newtheorem{thm}{Theorem}
\newtheorem*{thm-nonumber}{Theorem}

\newtheoremstyle{mytheorem}
{\topsep}
{\topsep}
{\itshape}
{0pt}
{\bfseries}
{.}
{ }
{\thmname{#1}.\thmnumber{#2}\thmnote{ (#3)}}
\theoremstyle{mytheorem}

\newtheorem{asm}{A}
\labelformat{asm}{A.\theasm}

\title{Empirical Risk Minimization for Time Series: \\ Nonparametric Performance Bounds for Prediction}
\author{Christian Brownlees$^*$ \and Jordi Llorens-Terrazas\thanks{%
		Universitat Pompeu Fabra and Barcelona Graduate School of Economics. \newline
		e-mail: \texttt{christian.brownlees@upf.edu}, \texttt{jordi.llorens@upf.edu}. \newline
		We would like to thank Gabor Lugosi, Geert Mesters, David Rossell, Francesco Violante and Piotr Zwiernik for providing numerous helpful comments.
		We would also like to thank participants at the finance research unit seminar, univeristy of Copenhagen (May 10, 2021) and finance and insurance seminar, CREST (May 27, 2021). 
		Christian Brownlees acknowledges support from the Spanish Ministry of Science and Technology (Grant MTM2012-37195)
		and the Spanish Ministry of Economy and Competitiveness through the Severo Ochoa Programme for Centres of Excellence in R\&D (SEV-2011-0075).}}

\date{\normalsize First Draft: July 1, 2021 \\ This Draft: \today}
	
\begin{document}
\maketitle
\begin{abstract}
Empirical risk minimization is a standard principle for choosing algorithms in learning theory. 
In this paper we study the properties of empirical risk minimization for time series.
The analysis is carried out in a general framework that covers different types of forecasting applications encountered in the literature. 
We are concerned with 1-step-ahead prediction of a univariate time series generated by a parameter-driven process.
A class of recursive algorithms is available to forecast the time series.
The algorithms are recursive in the sense that the forecast produced in a given period is a function of the lagged values of the forecast and of the time series.
The relationship between the generating mechanism of the time series and the class of algorithms is unspecified.
Our main result establishes that the algorithm chosen by empirical risk minimization achieves asymptotically the optimal predictive performance that is attainable within the class of algorithms.


\end{abstract}
\setstretch{1.50}

\clearpage

\section{Introduction}

Empirical risk minimization is a standard principle for choosing algorithms in learning theory \citep{Vapnik:Chervonenkis:1971,Devroye:Gyorfi:Lugosi:1996}.
Simply put, empirical risk minimization consists in choosing the algorithm that minimizes the empirical risk.
One of the main goals of learning theory is to establish bounds on the predictive performance
of the algorithm that minimizes the empirical risk relative to the optimal performance attainable in a given class of algorithms. 
A key feature of learning theory is its nonparametric nature, 
in the sense that performance bounds are typically obtained under the assumption that the generating mechanism of the data is unknown.
Despite the fact that empirical risk minimization is a general principle and widely applicable, 
the majority of contributions in this area focus on the analysis of i.i.d.~data.

In this paper we study empirical risk minimization for time series.
Our analysis is carried out in a general framework that allows to study different types of forecasting applications.
We are concerned with 1-step-ahead prediction of a univariate stationary time series generated by a (possibly nonlinear) parameter-driven process.
The class of processes we entertain is fairly broad and it includes linear state space and stochastic volatility models.
A class of recursive algorithms is available to predict the time series.
The algorithms are recursive in the sense that the forecast produced in a given period is a function of the lagged values of the forecast and the time series.
The class we consider is inspired by threshold models \citep{Tong:1990} and it includes as special cases the prediction formulae/filters of ARMA and GARCH models.
The prediction accuracy of the forecasts is measured by a loss function in the Bregman class \citep{Bregman:1967,Banerjee:2005,Laurent:Rombouts:Violante:2013,Patton:2020}, 
which includes the loss functions typically used for the estimation of ARMA and GARCH models.
Our analysis is nonparametric in the sense that the relationship between the data generating mechanism of the time series and the class of algorithms is unspecified. 

The main result of this paper consists in establishing an oracle inequality that provides non-asymptotic guarantees on the predictive performance of empirical risk minimization.
The oracle inequality implies that empirical risk minimization is consistent, in the sense that the algorithm chosen by empirical risk minimization
achieves asymptotically the optimal predictive performance that can be attained within the class of algorithms considered.
In particular, our result implies that ARMA/GARCH prediction based on the standard Gaussian maximum likelihood estimator achieves the optimal predictive performance even 
when the conditional mean/conditional variance equation of the model is misspecified.

The main result is illustrated by a number of basic applications of the general framework.
We consider forecasting time series generated by an AR(1) plus noise model as well as a stochastic volatility model.
In the case of stochastic volatility we consider forecasting on the basis of returns or a realized volatility measure.
%
%
Last, as a side result, we show that our class of algorithms can be interpreted as the solution of a sequential optimization problem that consists in minimizing an appropriately defined measure of tracking error of the algorithm.

The main result follows from five intermediate propositions.
We begin by establishing existence of moments and strong mixing conditions of a joint process that includes the time series and the algorithm (Proposition \ref{prop:mom_and_mixing}).
Importantly, the strong mixing coefficients are bounded by a function with geometric decay uniformly over the class of algorithms.
Next we establish a general inequality that states that the performance of empirical risk minimization can be controlled by the sum of two quantities (Proposition \ref{prop:generalized_gabor}). 
The first is the supremum of an average of differences between conditional and unconditional expectations and the second is the supremum of the empirical process associated with the prediction loss of the algorithm.
The first term is bounded using an inequality from Ibragimov (Proposition \ref{prop:cond-uncond}).
The second term is bounded using a covering argument (Proposition \ref{prop:covering}) and a concentration inequality for strong mixing processes (Proposition \ref{prop:conc}).

Proposition \ref{prop:mom_and_mixing} contains the main novel idea of the paper.
The result builds upon the literature on nonlinear time series models and Markov chains \citep{Bougerol:Picard:1992,Lanne:Saikkonen:2005,Francq:Zakoian:2006,Meitz:Saikkonen:2008,Kristensen:2009}.
The novelty with respect to the literature consists in using Markov chain theory to establish moment and dependence properties of an \emph{algorithm}, as opposed to a \emph{model}.
More precisely, the strategy consists in embedding the time series and the algorithm in what we name a companion Markov chain.
We then show that the companion Markov chain is $V$-geometric ergodic,
which implies existence of moments and strong mixing of the time series and the algorithm \citep{Meyn:Tweedie:1993}.
The uniform bound on the strong mixing coefficients is established using results by \citet{Roberts:Rosenthal:2004}.
This approach is motivated by the fact that while it can be challenging to characterize the moment and dependence properties of general nonlinear processes,
a number of tools are available to establish these properties for Markov nonlinear processes \citep{Carrasco:Chen:2002}.
We emphasize that the result does not hinge on the approximation properties of the class of algorithms.

Four remarks are in order before we proceed.
%
First, empirical risk minimization has a number of analogies with quasi-maximum likelihood estimation for ARMA/GARCH models.
Important research in the area includes \citet{Lee:Hansen:1994}, \citet{Lumsdaine:1996}, \citet{Ling:McAleer:2003}, \citet{Francq:Zakoian:2004}, \citet{Kristensen:Rahbek:2005} and \citet{Straumann:Mikosch:2006}.
Contributions in this literature typically assume that the conditional mean/conditional variance equation of the time series is known whereas the innovation distribution is not.
Interest then lies in estimating the parameters of the conditional mean/conditional variance equation.
The main difference with these contributions is that in this paper the relationship between the data generating process and the algorithm is unspecified.
In particular, the class of algorithms may not contain the conditional mean/conditional variance of the time series.

Second, empirical risk minimization for time series is closely related to M-estimation for dependent data.
Classic references in this area include \citet{Gallant:White:1988} and \citet{Potscher:Prucha:1997}, which develop general theory on the basis of fairly high-level assumptions.
\citet{Gallant:White:1988} rely, among other requirements, on uniform NED and dominance conditions on the objective function of M-estimation.
We remark that checking that these conditions hold is not always straightforward.
Instead, in this paper we rely on primitive assumptions to establish that conditions akin uniform NED and dominance hold.

Third, empirical risk minimization is close in spirit to nonparametric time series modeling.
Important research in this area includes \citet{Pagan:Schwert:1990}, \citet{Masry:Tjostheim:1995}, \citet{Hardle:Tsybakov:1997} and \citet{Linton:Mammen:2005}.
This literature focuses on developing nonparametric estimation techniques to estimate the time series model that has generated the data. 
The main difference with these contributions is that we are not concerned with estimating the time series model that has generated the data
and that interest solely lies in choosing an optimal algorithm for prediction within a given class.

Fourth, this paper contributes to the literature on empirical risk minimization for dependent data.
Besides a number of notable contributions, this literature is not extensive.\footnote{We remark that nontrivial technical challenges arise with dependent data. \cite{Mendelson:2015} argues that some of the standard techniques used in learning theory cannot be extended beyond i.i.d.~and bounded data setup.}
Two closely related contributions are \citet{Jiang:Tanner:2010} and \citet{Brownlees:Gudmundsson:2021}, which study empirical risk minimization for regression.
The class of algorithms considered in these papers depends on a finite number of lags of the time series.
In such a setting it is typically straightforward to obtain the dependence properties of the joint system composed of the times series and the algorithm.
The strategy adopted in both papers consists in assuming that the time series is strong mixing and then applying 
standard results for functions of strong mixing processes to obtain that the joint process is also strong mixing.
Such an approach is not viable in the framework of this paper.
In our setup forecasts depend on the entire past history of the time series. In this case standard results for functions of strong mixing processes do not provide useful results.

%
%

The rest of the paper is structured as follows.
Section \ref{sec:framework} introduces the basic framework.
Section \ref{sec:erm} presents empirical risk minimization and the main result of the paper.
Section \ref{sec:app} contains applications.
Section \ref{sec:proof} outlines the proof of the main result.
Concluding remarks follow in Section \ref{sec:end}.
Proofs are in the Appendix.

\section{Basic Definitions and Assumptions}\label{sec:framework}

\paragraph{Data generating process.} 
We are concerned with 1-step-ahead prediction of a stationary time series $\{ Y_t , t\geq 0 \}$ generated by a parameter-driven process.
The process $\{ Y_t , t\geq 0 \}$ takes values in $\mathcal Y \subseteq \mathbb R$ and is defined as $Y_0 = y \in \mathcal Y$, and
\begin{eqnarray}
	Y_t &=& g_{y1}(H_t) + g_{y2}(H_t) \epsilon_{Y\,t}~, \hspace{1em} t \geq 1~, \label{eqn:dgp:1} 
\end{eqnarray}
where $\{ H_t , t \geq 0 \}$ is a hidden process, $\{ \epsilon_{Y\,t} , t\geq 1\}$ is an i.i.d.~sequence of random variables and $g_{y1}$ and $g_{y2}$ are Borel-measurable real functions.
The process $\{ H_t , t \geq 0 \}$ takes values in $\mathcal H = \mbox{int}(\mathcal Y)$ and is defined as $H_0 = h \in \mathcal H$, and
\begin{eqnarray}
	H_t &=& g_{h1}(H_{t-1}) + g_{h2}(H_{t-1})\epsilon_{H\,t}~, \hspace{1em} t \geq 1~, \label{eqn:dgp:2} 
\end{eqnarray}
where $\{ \epsilon_{H\,t} , t\geq 1\}$ is an i.i.d.~sequence of random variables and $g_{h1}$ and $g_{h2}$ are Borel-measurable real functions.
We remark that in our framework, depending on the application, the target time series $\{ Y_t , t \geq 0 \}$ may denote some appropriate transformation of the data.
For example, in volatility forecasting using stock returns, where interest lies in predicting the 1-step-ahead scale of stock returns, the time series $\{ Y_t , t \geq 0 \}$ may be defined as the squared return process.


The data generating process satisfies the following set of assumptions.

\begin{asm}[Data generating process]\label{asm:dgp}~

	\begin{enumerate}[(i)]
	\item The functions $g_{h1}$ and $g_{h2}$ are bounded on bounded subsets of $\mathbb R$. There exist positive constants $a_h$ and $b_h$ such that
	$|g_{h1}(h)| \leq a_h |h| + o(|h|)$ as $|h| \rightarrow \infty$ and 
	$|g_{h2}(h)| \leq b_h |h| + o(|h|)$ as $|h| \rightarrow \infty$.
	The function $g_{h2}$ satisfies $\inf_h {|g_{h2}(h)|} > 0$.

	\item The functions $g_{y1}$ and $g_{y2}$ are bounded on bounded subsets of $\mathbb R$.
	There exist positive constants $C_{y1}$ and $C_{y2}$ such that $|g_{y1}(h)| \leq C_{y1}|h| $ and $|g_{y2}(h)| \leq C_{y2}(1\vee|h|)$.
	The function $g_{y1}$ satisfies $\inf_h g_{y1}(h)\geq 0$ when $\mathcal Y= \mathbb R_+$.
	The function $g_{y2}$ satisfies $\inf_h g_{y2}(h)>0$.

	\item The random vector $\epsilon_t=(\epsilon_{Y\,t},\epsilon_{H\,t})'$ has a distribution that is absolutely continuous with respect to Lebesgue measure on $\mathbb R^2$
	and is supported on $(\underline \epsilon, \infty)^2$ with $\underline \epsilon = -\infty$ when $\mathcal Y = \mathbb R$ and $\underline \epsilon = 0$ when $\mathcal Y = \mathbb R_+$.
	The joint density $\phi(\epsilon_t)$ of the random vector $\epsilon_t$ satisfies $  \phi(\epsilon_t) =  \phi_Y(\epsilon_{Y\,t})\phi_H(\epsilon_{H\,t}) $, 
	where $\phi_Y$ and $\phi_H$ are densities that are bounded away from zero on compact subsets of $(\underline \epsilon, \infty)$.
	The random variables $\epsilon_{Y\,t}$ and $\epsilon_{H\,t}$ satisfy $\mathbb E \epsilon_{H\,t}^{2r_m}  < \infty$,
	$\mathbb E \epsilon_{Y\,t}^{2r_m} < \infty$ for some $r_m\geq 6$. The random variable $\epsilon_{Y\,t}$ satisfies $\mathbb E \left(\log\epsilon_{Y\,t}\right)^{2r_m} < \infty$ when $\mathcal Y = \mathbb R_+$.
	\item The condition $\mathbb E ( a_h + b_h | \epsilon_{H\,t} | )^{2r_m} < 1$ holds.
\end{enumerate}
\end{asm}

Assumption \ref{asm:dgp} is similar to standard assumptions used to establish geometric ergodicity of nonlinear time series models \citep{Masry:Tjostheim:1995,Lu:Jiang:2001,Lanne:Saikkonen:2005,Meitz:Saikkonen:2008}
and it allows for a fairly broad class of parameter-driven processes.
Note that the $\{ Y_t, t\geq 0 \}$ process can take values on either $\mathcal Y=\mathbb R$ or $\mathcal Y=\mathbb R_+$ (assumptions differ slightly depending on these two cases).
This allows us to cover different types of forecasting applications encountered in the literature.

Assumption \ref{asm:dgp}$(i)$ is similar to Assumption 3.2 in \cite{Masry:Tjostheim:1995} and it implies that \eqref{eqn:dgp:2} is dominated asymptotically by a stable linear model.
As \cite{Masry:Tjostheim:1995} emphasize, such a requirement is mild, 
since functions that grow everywhere faster than a stable linear model are nonstationary. 

Assumption \ref{asm:dgp}$(ii)$ allows for a fair amount of flexibility in equation \eqref{eqn:dgp:1}.
In particular, it requires $|Y_t|$ to be bounded from above by a linear function of $|H_t|$.

Assumption \ref{asm:dgp}$(iii)$ imposes conditions on the random variables $\epsilon_{H\,t}$ and $\epsilon_{Y\,t}$ that are, for the most part, analogous to standard conditions used in the literature.
The less standard requirement is assuming, when $\mathcal Y=\mathbb R_+$, that $2r_m$ moments of $\log \epsilon_{Y\,t}$ exist.
This guarantees that the moments of some of the loss functions considered in this paper are finite.
The assumption is relatively mild. For example, it allows for distributions with density bounded from above in a neighborhood of zero (e.g.~the exponential) as well as certain distributions with unbounded density (e.g.~chi-square with one degree of freedom).
Finally, the independence assumption between $\epsilon_{H\,t}$ and $\epsilon_{Y\,t}$ may be relaxed at the expense of more tedious proofs.

Assumption \ref{asm:dgp}$(iv)$ is a stability condition analogous to the one assumed in \citet{Masry:Tjostheim:1995} or \citet{Lanne:Saikkonen:2005}.

\paragraph{Algorithms.}
A class of recursive algorithms indexed by $ \theta \in \Theta \subset \mathbb R^p$ and denoted by $\{ f_{\theta\,t} , t \geq 0 \}$ is available to predict 1-step-ahead the time series $\{ Y_t, t\geq 0\}$. 
The process $\{ f_{\theta\,t} , t \geq 0 \}$ takes values in $\mathcal F \subseteq \mathbb R$ and is defined as $f_{\theta\,0} = f \in \mathcal F$ 
and
\begin{equation}\label{eqn:forecast:threshold}
	f_{\theta\,t}
	= \sum_{k=1}^K ( \alpha_{0\,k} +  \alpha_{1\,k}  Y_{t-1}  + \beta_{1\,k }f_{\theta\,t-1} ) \mathbbm 1_{t-1\,k}~, \hspace{1em} t \geq 1~,
\end{equation}
where $\theta = (\alpha_{0\,1},\ldots,\alpha_{0\,K},\alpha_{1\,1},\ldots,\alpha_{1\,K},\beta_{1\,1},\ldots,\beta_{1\,K})'$ with $K=p/3$, 
$\mathbbm 1_{t\,k} = \mathbbm 1_{\{Y_{t} \in \mathcal Y_k \}}$ and $\{ \mathcal Y_1, \ldots, \mathcal Y_K \}$ is a known partition of $\mathcal Y$ made of $K$ sets referred to as regimes.
The partition is of the form $\{ (r_1, r_2 ) , [r_2,r_3) , \dots , [r_K,\infty) \}$ with $-\infty = r_1 < r_2 < \dots < r_K < \infty$ when $\mathcal Y = \mathbb R$ and
$\{ [r_1, r_2 ) , [r_2,r_3) , \dots , [r_K,\infty) \}$ with $0 = r_1 < r_2 < \dots < r_K < \infty$ when $\mathcal Y = \mathbb R_+$.
The parameter vector $\theta$ is referred to as a prediction rule. 
We remark that the class of prediction algorithms defined in \eqref{eqn:forecast:threshold} corresponds to the class of 1-step-ahead prediction formulae induced by the self-exciting threshold autoregressive moving average model (SETARMA) \citep{Tong:1990}.
As it is customary in the learning literature, the relationship between $Y_t$ and $f_{\theta\,t}$ is unspecified
and \eqref{eqn:forecast:threshold} is simply an algorithm to predict $Y_t$.

The class of algorithms satisfies the following set of assumptions.

\begin{asm}[Algorithms]\label{asm:algo}
	$(i)$ The set $ \Theta \subset \mathbb R^p$ with $p=3K$ is nonempty and such that $\Theta \subseteq [\underline \alpha_0,\overline \alpha_0]^K \times [\underline \alpha_1,\overline \alpha_1]^K \times [0,\overline \beta_1]^K $
	with $\underline \epsilon < \underline \alpha_0 < \overline \alpha_0 < \infty$, $0<\underline \alpha_1<\overline \alpha_1 < \infty$ and $\overline \beta_1 < 1$. 
	$(ii)$ The number of regimes $K$ satisfies $K < (r_m - 2)/3 $.
\end{asm}

The process $\{ f_{\theta\,t}, t\geq 0 \}$ takes values in $\mathcal F=\mathbb R$ when $\mathcal Y = \mathbb R$ and $\mathcal F = [\underline \alpha_0,\infty)$ when $\mathcal Y = \mathbb R_+$.

Assumption \ref{asm:algo}$(i)$ is mild and imposes constraints on the class of prediction rules $\Theta$ that are analogous to standard constraints imposed in the analysis of quasi-maximum likelihood estimators of ARMA and GARCH models \citep{Francq:Zakoian:2010}. 
We remark that when $\mathcal Y = \mathbb R$ the constraint $\beta_{1\,k} \in [0,\overline \beta_1]$ may be relaxed to $\beta_{1\,k} \in [-\overline \beta_1,\overline \beta_1]$ at the expense of more tedious proofs. 

Assumption \ref{asm:algo}$(ii)$ states that the size of the class of prediction rules is bounded by a linear function of the number of moments of $\epsilon_{Y\,t}$ and $\epsilon_{H\,t}$. 

\paragraph{Loss function.} The prediction accuracy of the algorithm is measured by a loss function that belongs to the Bregman class. 
Let $\psi : \mathcal S \rightarrow \mathbb R$ be a strictly convex and continuously differentiable function defined over a convex set $\mathcal S \subseteq \mathbb R$.
Then, the Bregman loss associated with $\psi$ for predicting $Y_t$ with $f_{\theta\,t}$ is defined as 
\begin{equation}\label{eqn:loss}
	L( Y_t , f_{\theta\,t} ) = \psi(Y_t) - \psi( f_{\theta\,t}) - \nabla \psi( f_{\theta\,t}) (Y_t - f_{\theta\,t} ) ~.
\end{equation}
The Bregman class is a fairly large and tractable family of losses.
In particular, the log-likelihood of random variables in the regular exponential family can be expressed as the (negative) sum of Bregman losses (up to a constant term) \citep{Banerjee:2005}.
Thus, the Bregman class includes the standard loss functions used for quasi-maximum likelihood estimation of time series models.

In this paper we focus exclusively on losses that satisfy the following condition.

\begin{cond}[Bregman]\label{cond:bregreg}
	The loss $L$ is such that
	(i) $Y_t \in \mathcal S$ a.s.~for all $t\geq 0$,
	(ii) $\sup_{\theta\in\Theta} \mathbb E ( L(Y_t,f_{\theta\,t}) )^{r_m} < \infty$ for all $t\geq 0$ and
	(iii) $L( f_{\theta_1\,t} , f_{\theta_2\,t} ) \leq C_\psi ( f_{\theta_1\,t} - f_{\theta_2\,t} )^2$ 
	a.s.~for all $t\geq 0$, for any $\theta_1, \theta_2 \in \Theta$ and for some positive constant $C_\psi$.
\end{cond}

Table \ref{tbl:losses} contains a number of Bregman losses that satisfy Condition \ref{cond:bregreg} given Assumptions \ref{asm:dgp} and \ref{asm:algo}.
We remark that when $\mathcal Y=\mathbb R$ only the first two losses are admissible whereas when $\mathcal Y = \mathbb R_+$ all the losses in the table are allowed. 
The table contains both well known and lesser known loss functions.\footnote{The random variables listed in Table \ref{tbl:losses} are all the random variables in the natural exponential family with quadratic variance function and unbounded support \citep{Morris:1982}.}
The table includes the loss that corresponds to the log-likelihood of the Gaussian (with known variance) with respect to the mean parameter, which is the classic square loss.
This loss function is typically used for maximum likelihood estimation of ARMA models.
The table also contains the loss associated with the log-likelihood of the NEF-GHS (with known number of convolutions) with respect to the natural parameter \citep{Morris:1982}. 
The NEF-GHS is a flexible distribution taking values on the real line that allows for skewness and higher kurtosis than the Gaussian. 
In addition, compared to the square loss it has the advantage that the $r$-th moment of the loss only requires the existence of the $r$-th moment of its arguments.
To the best of our knowledge this loss function/distribution has not been used extensively in the time series literature. 
Next, the table includes the loss associated with the log-likelihood of the gamma (with known shape) with respect to the mean parameter,
which in the volatility forecasting literature is known as the QLIKE loss \citep{Patton:2011}.\footnote{The standard definition of the QLIKE is $L(Y_t,f_{\theta\,t})=Y_t/f_{\theta\,t}+\log f_{\theta\,t}$. This is equivalent to our definition for optimization purposes with respect to $\theta$.}
We recall that by appropriately constraining the shape parameter, the gamma distribution nests the exponential and chi-square distributions.
This loss function is typically used for maximum likelihood estimation of MEM \citep{Engle:Gallo:2006}, ACD models \citep{Engle:Russel:1998} and GARCH models.
Finally, the table includes the losses associated with the log-likelihoods of the Poisson and negative binomial (with known number of failures) with respect to the mean parameter.\footnote{We follow the convention that $0\log 0 = 0$, hence $\mbox{dom}(\psi)=\mathbb R_+$ in both cases.}
These loss functions are typically used for maximum likelihood estimation of dynamic models for count data \citep{Agosto:Cavaliere:Kristensen:Rahbek:2016,Davis:etal:2021}.
We remark that our framework does not allow for $\{ Y_t, t\geq 0\}$ to take values on a countable set. 
That said, these losses satisfy our regularity conditions and may be used for empirical risk minimization.
The analysis of empirical risk minimization when $\{ Y_t, t \geq 0 \}$ takes values in a countable set can be carried out using the same strategy developed 
in this paper, but some of the proofs would differ.

\begin{table}\caption{Regular Bregman Losses}\label{tbl:losses}
\medskip
\resizebox{\textwidth}{!}{
\begin{tabular}{l l l l}
\hline
$\mathcal S$ & $\psi(u)$ & $L(u,v)$ & Log-likelihood \\
\hline
\hline
$\mathbb R$ & $u^2$ & $(u-v)^2$ & Gaussian \\
$\mathbb R$ & $u \tan^{-1}(u) - {1\over 2}\log(1+u^2)$ &$u \left[\tan^{-1}(u)-\tan^{-1}(v)\right] + {1\over 2}\log{ 1+v^2 \over 1 + u^2 }$ & NEF-GHS  \\
$\mathbb R_{++}$ & $- \log u $ & ${ u\over v} - \log {u \over v} - 1$ & Gamma\\
$\mathbb R_+$ & $u\log u - u$ & $u \log{u\over v} -(u-v)$ & Poisson \\
$\mathbb R_+$ & $u  \log {u \over 1 + u} -\log(1+u)$ & $u \log {u\over v}+( 1+u)\log{1+v\over 1+u}$ & Negative Binomial\\
\hline
\end{tabular}
}

\medskip
{\scriptsize The table lists Bregman losses that satisfy Condition \ref{cond:bregreg} given Assumptions \ref{asm:dgp} and \ref{asm:algo}.}
\end{table}

\paragraph{Dominating process.}
We introduce a dominating process $\{d_{\theta\,t}, t\geq 0\}$ that plays a key role in the theoretical analysis of this paper.
This process bounds the absolute difference between the forecast processes associated with two different prediction rules.
The process $\{ d_{\theta\,t}, t \geq 0 \}$ takes values in $\mathcal D = [1,\infty)$ and is defined as $d_{\theta\,0} = d \in \mathcal D$ and
\begin{align}\label{eqn:bounding_process}
	d_{\theta\,t} = 1 + |Y_{t-1}| + |f_{\theta\,t-1}| + \overline \beta_1 d_{\theta\,t-1} ~, \hspace{1em} t \geq 1 ~.
\end{align}
As it is established in one of the intermediate results of this paper, this process has the property that for any $\delta\in(0,1]$
and for any $\theta, \dot \theta \in \Theta$ such that $\| \theta - \dot \theta \|_2 \leq \delta$ it holds that $ | f_{\theta\,t} - f_{\dot \theta\,t} | \leq \delta d_{\dot \theta\,t}$ for all $t \geq 0 $.
This property and the generalized triangular equality for Bregman losses imply that
\begin{equation}\label{eqn:domination}
	L( Y_t , f_{\theta\,t} ) \leq L( Y_t , f_{\dot \theta\,t} ) + \delta C_\psi ( d_{\dot \theta\,t}^2 +  2|Y_t-f_{\dot \theta\,t}| d_{\dot \theta\,t} ) ~, \hspace{1em} t\geq 0~.
\end{equation}

\section{Empirical Risk Minimization}\label{sec:erm} 

We are interested in choosing a prediction rule $\theta$ from a sequence of ``in-sample'' observations to forecast 1-step-ahead a sequence of ``out-of-sample'' observations.
The sequences of in-sample and out-of-sample observations are respectively defined as $\{Y_1,\ldots,Y_T\}$ and $\{ Y_{T+1}, \ldots. Y_{T+M} \}$.
The number of out-of-sample observations is such that $M = \lceil \gamma T \rceil$ for some $\gamma>0$.

The accuracy of a prediction rule $\theta$ is measured by the out-of-sample 1-step-ahead conditional risk, which is defined as 
\begin{equation}\label{eqn:accuracy}
	R( \theta ) = \mathbb E\left[ \left. {1\over M} \sum_{t=T+1}^{T + M}  L( Y_{t},f_{\theta\,t}) \right| Y_T , \ldots , Y_1 \right] ~.
\end{equation}
A natural strategy for choosing a prediction rule $\theta$ consists in picking the one that minimizes the in-sample 1-step-ahead empirical risk.
The empirical risk minimizer (ERM) is defined as
\begin{equation}\label{eqn:erm}
\hat{\theta} \in \arg \min_{\theta \in \Theta } R_T(\theta)~, \text{ where } R_T(\theta) = {1 \over T} \sum_{t=1}^{T} L( Y_{t} , f_{\theta\,t} )  ~.
\end{equation}
If more than one prediction rule achieves the minimum we may pick one arbitrarily.
In \eqref{eqn:accuracy} and \eqref{eqn:erm} we remark that $f_{\theta\,1}$ is computed using $Y_0=y$ and $f_{\theta\,0} = f$ that are fixed, known and that do not depend on $\theta$.\footnote{The initial value $Y_0$ can be a pre-sample observation assumed to be fixed or a fixed value set at the outset of the analysis. Note that when $Y_0$ is a pre-sample observation then the empirical risk in \eqref{eqn:erm} can be thought of as the analog of the conditional log-likelihood of $\theta$ given $Y_0$.}

One of the goals of learning theory is to establish a bound on the performance of the ERM relative to the optimal risk that can be achieved within the class of prediction rules considered. 
We measure the accuracy of the ERM on the basis of the conditional out-of-sample risk, which is defined as
\begin{equation}\label{eqn:risk:erm}
	R( \hat{\theta} )  =  \mathbb E \left[ \left. {1\over M} \sum_{t=1}^{T+M} L(Y_{t},\hat f_{t}) \right| Y_T , \ldots , Y_1 \right] ~,
\end{equation}
where $\hat f_{t} = f_{\hat{\theta}\,t}$.
The performance measure in \eqref{eqn:risk:erm} can be interpreted as the out-of-sample conditional risk of the ERM obtained from the in-sample observations.
The following theorem establishes such a bound and is our main result.

\begin{thm}\label{thm:erm}
	Suppose Assumptions \ref{asm:dgp} and \ref{asm:algo} are satisfied.
	Then there exists a constant $\sigma^2$ such that, for all $T$ sufficiently large, we have that 
	\[
		R(\hat \theta) \leq \inf_{\theta \in \Theta} R(\theta) + 47 \sigma \sqrt{p \log T \over T} 
	\]
	holds at least with probability $1- \log^{-1} T - o( \log^{-1} T )$ as $T \rightarrow \infty$.
\end{thm}

The inequality in Theorem \ref{thm:erm} is commonly referred to as an \emph{oracle inequality},
and it provides non-asymptotic guarantees on the performance of the ERM.\footnote{We remark that 
it is possible to obtain explicit bounds for the minimum $T$ and for the probability of the oracle inequality.
Moreover, it is straightforward to see from the intermediate results of this paper that it is possible to sharpen the rate of the probability upper bound of the oracle inequality as well as the absolute constant. 
However, we have not pursued this and we have solely focused on recovering the ``classic'' rate of convergence $\sqrt{\log T / T}$.}
The constant $\sigma^2$ is application-specific and may be interpreted as an upper bound for the long run variance of the loss process. We define the constant precisely in Proposition \ref{prop:conc}.
The rate of convergence $\sqrt{\log T/T}$ is sometimes referred to as the classical rate of convergence of empirical risk minimization in the learning literature for classification with i.i.d.~data \citep[Ch.~12]{Devroye:Gyorfi:Lugosi:1996}.
The theorem implies that in our framework the ERM is consistent with respect to the class of prediction rules $\Theta$, 
meaning that $|R(\hat\theta) - \inf_{\theta\in\Theta} R(\theta) | \stackrel{p}{\rightarrow} 0$. 
In other words, the ERM achieves asymptotically the optimal forecasting performance attainable within the class of algorithms considered.
We emphasize that the existence of an optimal prediction rule $\theta^* = \arg \min_{\theta\in\Theta} R(\theta)$ is not required by the theorem. 

Let us highlight the stability conditions required by the theorem.
These are the stability of the data generating process (Assumption \ref{asm:dgp}.$(iv)$) 
and the condition $\beta_{1k} \leq \overline \beta_1 < 1$ for $k=1,\ldots,K$ on the class of algorithms (Assumptions \ref{asm:algo}.$(i)$). 
We remark that this latter condition differs from the standard stability conditions of SETARMA models and, as a consequence, of ARMA and GARCH models.
For simplicity, we discuss this in the case of an ARMA(1,1). 
For an ARMA(1,1) model a necessary condition for stability is $\alpha_{1\,1} + \beta_{1\,1} < 1$ (using the notation of this work).
However, in our framework this constraint is not required by the class of algorithms that corresponds to the ARMA(1,1).

\subsection{Additional Discussion}\label{sec:discuss}

A number of additional remarks on the framework of this paper are in order.

Our analysis studies the properties of the ERM when the time series is generated by a parameter-driven process.
Clearly, an observation-driven process may be entertained instead.
In this case, the analysis of the performance of the ERM can be carried out using the same strategy developed in this paper.
However, some of the proofs will differ and we leave the analysis of this case for future research.

The class of recursive algorithms we entertain is fairly flexible and builds upon the class of threshold models that have a well established tradition in the time series literature.
We remark that our results may be extended to alternative classes of recursive algorithms and do not inherently depend on the functional form of the algorithmic class we consider in this paper.
%
In particular, our analysis does not require the class of algorithms 
to have special approximation properties or to include the optimal 1-step-ahead forecast associated with the data generating process and the loss function.
What is key in our framework is that, loosely speaking, the algorithms ``forget the past sufficiently fast''.

Instead of comparing the performance of the ERM against the optimal risk attainable in the class, one may wish to compare against the risk of the optimal 1-step-ahead forecast.
For loss functions in the Bregman class the optimal 1-step-ahead forecast is the conditional mean (assuming it exists) \citep{Banerjee:2005}.
Thus, the risk of the optimal 1-step-ahead forecast may be defined as
\[
	R^* = \mathbb E \left[ \left. {1 \over M} \sum_{t=T+1}^{T+M} L( Y_t , \mu_{t} ) \right| Y_T, \ldots, Y_1\right],
\]
where \( \mu_{t} = \mathbb E( Y_t | Y_{t-1}, \ldots, Y_1) \) for $t > 1$ and $\mu_1 = \mathbb E(Y_1)$.
The performance of the ERM relative to the risk of the optimal 1-step-ahead foreast may be expressed as
\[
	R( \hat \theta ) - R^* = \left[ \inf_{\theta \in \Theta} R( \theta ) - R^* \right] + \left[ R( \hat \theta ) - \inf_{\theta \in \Theta} R( \theta ) \right].
\]
The first term is called the approximation error and the second term is called the estimation error \citep[Ch.~12]{Devroye:Gyorfi:Lugosi:1996}.
Notice that oracle inequalities control the estimation error.
The approximation error is typically difficult to control, especially in a data dependent setting.
There are a number of contributions that, in some sense, attempt to control the approximation error \citep{Nelson:1992}.
%
%
%
In general, the analysis of the approximation error requires additional assumptions.
For this reason learning theory typically focuses on studying the estimation error, which is the approach pursued in this paper.

\section{Applications}\label{sec:app}

We illustrate our framework with a number of applications.
For simplicity, we shall always consider forecasting using the class of algorithms defined in \eqref{eqn:forecast:threshold} when the number of regimes is set to $K=1$.\footnote{The detailed analysis of these applications is carried out in the Online Appendix.} 

\paragraph{Forecasting an AR(1) plus noise.}\label{ex:lss}
Consider the AR(1) plus noise model given by $Y_0=y \in \mathbb R$, $H_0=h \in \mathbb R$ and
\begin{eqnarray*}\label{eqn:ar1}
	Y_t &=& H_t + \epsilon_{Y\,t} ~, \\
	H_t &=& \mu_H + \varrho ( H_{t-1} - \mu_H) + \epsilon_{H\,t} ~,
\end{eqnarray*}
for $t\geq 1$, where $\{ \epsilon_{Y\,t} , t\geq 1\}$ and $\{ \epsilon_{H\,t} , t\geq 1\}$ are i.i.d.~sequences of Gaussian random variables and $\varrho \in [0,1)$.
The class of algorithms defined in \eqref{eqn:forecast:threshold} is used for forecasting.
Prediction accuracy is measured by the square loss.
Then, Assumptions \ref{asm:dgp} and \ref{asm:algo} are satisfied and Theorem \ref{thm:erm} holds.

We remark that in this application we have that the approximation error of the class of algorithms converges to zero when $T$ is large (when $\Theta$ is suitably chosen). 
The class of algorithms includes the steady state Kalman filter, which implies that the class of algorithms includes a forecast process that converges to the conditional mean of $Y_t$ given the past when $t$ is large. 


\paragraph{Forecasting volatility using returns.}\label{ex:sv}
Consider the stochastic volatility model for the return process $\{ r_t, t \geq 0 \}$ given by $r_0 \in \mathbb R$, $\sigma^2_0 \in\mathbb R_{++}$ and
\begin{eqnarray*}\label{eqn:sv1}
r_t &=& \sqrt{ \sigma^2_t }~z_t ~, \\
\log \sigma^2_t &=& \mu_H + \varrho (\log \sigma^2_{t-1} - \mu_H) + \eta_t ~, 
\end{eqnarray*}
for $t\geq 1$, where $\{ z_t , t\geq 1\}$ and $\{ \eta_{t} , t\geq 1\}$ are i.i.d.~sequences of Gaussian random variables and $\varrho \in (0,1)$.
It is straightforward to see that this model belongs to the class of data generating processes considered in this paper
(using $Y_t = r_t^2$, $H_t=\sigma^2_t$, $\epsilon_{Y\,t} = z_t^2$ and $\epsilon_{H\,t}=\exp(\eta_t)$).
The class of algorithms defined in \eqref{eqn:forecast:threshold} is used for forecasting (using $Y_t=r_t^2$), which corresponds to the 1-step-ahead prediction formula of the GARCH(1,1). 
Prediction accuracy is measured by the QLIKE loss.
Then, Assumptions \ref{asm:dgp} and \ref{asm:algo} are satisfied and Theorem \ref{thm:erm} holds.

We remark that in this application the ERM coincides with the Gaussian quasi-maximum likelihood estimator of the GARCH(1,1).
Despite the fact that the conditional variance equation implied the GARCH(1,1) is misspecified in this setting,
our main theorem implies that the ERM chooses an algorithm with asymptotically optimal predictive performance within the class of algorithms considered.


\paragraph{Forecasting volatility using realized volatility.}
Over the last two decades realized volatility measures have enhanced volatility prediction \citep{Andersen:Bollerslev:Diebold:Labys:2003}.
Realized volatility measures are precise estimators of the (latent) volatility computed from intra-daily stock prices.
Consider the (nonlinear) stochastic volatility model for the realized volatility process $\{ RV_t , t\geq 0\}$ given by $RV_0 \in \mathbb R_{++}$, $\sigma^2_0 \in\mathbb R_{++}$ and
\begin{eqnarray*}
	RV_t & = & \sigma^2_{t} \epsilon_{RV\,t} ~, \\
	\sigma^2_t & = & g_{h1}(\sigma^2_{t-1}) + g_{h2}(\sigma^2_{t-1}) \epsilon_{\sigma^2\,t}~,
\end{eqnarray*}
for $t\geq 1$, where $\{ \epsilon_{RV\,t} , t\geq 1\}$ and $\{ \epsilon_{\sigma^2\,t} , t\geq 1\}$ are i.i.d.~sequences of gamma random variables 
and $g_{y1}$ and $g_{y2}$ are Borel-measurable real functions that satisfy Assumption \ref{asm:dgp}.$(i)$.
We assume that $\epsilon_{RV\,t}$ is unit mean, which implies that $RV_t$ is a conditionally unbiased proxy for the volatility $\sigma^2_t$.\footnote{%
The realized volatility measurement error in the model is multiplicative. The analysis of this section can also be carried out in the case of an additive measurement error.}
The class of algorithms defined in \eqref{eqn:forecast:threshold} is used for forecasting, which corresponds to the 1-step-ahead prediction formula of the MEM(1,1) or ARMA(1,1).
Prediction accuracy is measured by the QLIKE or the square loss.
Then, Assumptions \ref{asm:dgp} and \ref{asm:algo} are satisfied and Theorem \ref{thm:erm} holds.

We remark that the model considered in this application should be interpreted as a reduced form approximation.
\cite{Meddahi:2003} has derived the discrete-time representation of the realized volatility process implied by a fairly general class of continuous-time stochastic volatility models commonly encountered in the literature.
These results imply that a more flexible framework than the one considered here is required to allow for such data generating processes,
and we have not pursued to accommodate this.\footnote{\cite{Meddahi:2002} establishes that for a general diffusion, the measurement error of realized volatility depends on the entire path of the spot volatility.
Our framework is consistent with the continuous-time stochastic volatility model with constant intra-daily volatility used in \citet{Patton:2011}.}

Theorem \ref{thm:erm} implies that the ERM achieves the optimal performance for realized volatility prediction.
However, interest typically lies in forecasting the latent volatility process $\{\sigma^2_t,t\geq 0\}$ rather than its noisy measurement. 
Building upon \citep{Hansen:Lunde:2006,Patton:2011} we establish further properties of the ERM.
We measure the accuracy of a prediction rule $\theta$ for predicting the volatility process $\{\sigma^2_t,t\geq 0\}$ using the out-of-sample 1-step-ahead conditional risk 
\[
	R_\mathrm{Vol}( \theta )  =  \mathbb E \left[ \left. {1\over M} \sum_{t=1}^{T+M} L(\sigma^2_{t},f_{\theta\,t}) \right| RV_T , \ldots , RV_1 \right] ~.
\]
The loss in predicting the volatility process $\{ \sigma^2_t , t\geq 0 \}$ satisfies
\begin{equation}\label{eqn:loss_iv}
	L( \sigma^2_t , f_{\theta\,t} ) = L( RV_t , f_{\theta\,t} ) + L( \sigma^2_t , RV_t ) - ( \sigma^2_t - RV_t ) (\nabla \psi( f_{\theta\,t}) - \nabla \psi(RV_t) ) ~,
\end{equation}
which follows from the generalized triangular equality for Bregman losses.
In \eqref{eqn:loss_iv} we have that the second term does not depend on the algorithm and the third term has a conditional expectation of zero given the past.
The decomposition in \eqref{eqn:loss_iv} and Theorem \ref{thm:erm} imply that 
\[
	\left| R_\mathrm{Vol}( \hat \theta ) - \inf_{\theta \in \Theta} R_\mathrm{Vol}( \theta ) \right| = \left| R( \hat \theta ) - \inf_{\theta \in \Theta} R( \theta ) \right| \stackrel{p}{\rightarrow} 0 ~.
\] 
Thus, the ERM based on the noisy realized volatility measure chooses an algorithm with optimal performance for volatility forecasting (within the class of algorithms) provided that the realized volatility measure is conditionally unbiased.

\paragraph{Recursive prediction as a solution of a sequential optimization problem.}
The class of algorithms defined in \eqref{eqn:forecast:threshold} was introduced without any justification other than its close connection to standard models used in the literature.
In this section we show that this class of algorithms may be motivated as the solution of a sequential optimization problem.
The analysis is inspired by the research by \citet{Creal:Koopman:Lucas:2013} and \citet{Harvey:2013} on GAS/DCS models 
and by \citet{Gijbels:Pope:Wand:1999,Harvey:Chakravarty:2008} on the relation between nonparametric estimators and time series models.\footnote{%
The 1-step-ahead prediction formula implied by GAS/DCS models is sometimes motivated as the approximate solution of a local estimation problem based on a generic (and sufficiently regular) likelihood function.
The class of algorithms we introduce for 1-step-ahead prediction can be interpreted as the exact solution of a local estimation problem based on a Bregman loss.}

Let $\{ f_{t} , t \geq 0 \}$ be defined as $f_{0} = f \in \mathcal F$ and 
\begin{align}\label{eqn:opt}
	f_{t} = \arg \min_{f\in\text{int}(\mathcal S)} Q_{t}(f) ~,
\end{align}
where $Q_{t}$ is the tracking error function defined as 
\begin{equation*}
	Q_{t}(f) = w_1 L(\bar f,f) + w_2 L(Y_{t-1},f) + w_3 L( f_{t-1} , f )  ~,
\end{equation*}
where $L$ denotes a loss in the Bregman family and $(w_1,w_2,w_3) = w \in \Delta^3$ with $\Delta^3$ denoting the 3-dimensional simplex.
The tracking error is a convex combination of the divergences with respect to the constant $\bar f$, the previous observation and the previous forecast.
If $f_0 = \bar f$,\footnote{We remark that the choice $f_0 = \bar f$ is made only for expository purposes, as it simplifies the notation in \eqref{eqn:kernel}.
This would imply that the initial value for the forecast process is determined by empirical risk minimization, which we do not cover in our framework. } it is straightforward to verify that 
\begin{align}\label{eqn:kernel}
Q_t(f) \propto \sum_{i=0}^{t-1} k\left({x_{t}-x_{t-i}\over h}\right)L\left(Y_{t-i-1},f\right) + \lambda L(\bar f, f)~,
\end{align}
where
$\{ x_t , t\geq 0 \}$ is a deterministic sequence defined as  $x_t=t$ for each $t \geq 0$, 
$k(u) = \exp(u)\mathbbm 1_{\{u\leq 0\}}$, $h = 1/\ln (w_3)$ and $\lambda = w_2^{-1} - \sum_{i=1}^t w_3^{i-1}$.
Thus, the tracking error can equivalently be thought of as the objective function of a local constant regression plus a regularization term that penalizes deviations from the constant $\bar f$.
The solution of this optimization problem is 
\begin{align}\label{eqn:sol:opt:1}
	f_{t}= w_1 \bar f + w_2 Y_{t-1}+ w_3 f_{t-1} ~, 
\end{align}
which coincides with the class of algorithms in \eqref{eqn:forecast:threshold} provided that $\theta = (\alpha_0,\alpha_1,\beta_1)'$ with $\alpha_0 = w_1 \bar f$, $\alpha_1 = w_2$ and $\beta_1 = w_3$. 
Note that empirical risk minimization may be interpreted as choosing the set of weights $w$ and the constant $\bar f$ in the objective function $Q_t$ that minimize the in-sample empirical prediction loss.

\section{Proof of Theorem \ref{thm:erm}}\label{sec:proof}

\subsection{Companion Markov Chain} 

The first step of our analysis consists in introducing a companion Markov chain associated with the process $\{ (Y_t,f_{\theta\,t} )' , t\geq0 \}$.
We recall a number of notions from Markov chain theory.
Notation and definitions are based on \cite{Meyn:Tweedie:1993}.
The discrete-time process $\{X_t, t \geq 0 \}$ is a time-homogeneous Markov chain with state space $\mathcal X \subseteq \mathbb R^d$ and equipped with a Borel $\sigma$-algebra $\mathcal B(\mathcal X)$
if for each $n\in\mathbb N$ there exists an $n$-step transition probability kernel 
$P^n_X : \mathcal X \times \mathcal B(\mathcal X) \rightarrow [0,1]$ such that \( P^n_X(x, A) = \mathbb P(X_{t+n} \in A \vert X_t = x) \) for all $t\in\mathbb Z_+$.
As customary, $P^1_X(x, A)$ is denoted by $P_X(x,A)$.
We use $\pi_X: \mathcal B(\mathcal X) \rightarrow [0,1]$ to denote the invariant measure of the Markov chain (assuming it exists), that is, the probability measure such that for each $A \in \mathcal B(\mathcal X)$ it holds that
\( \pi_X( A ) = \int_{\mathcal X} \pi_X(dx) P_X(x,A) \).

Define the companion Markov chain $\{ X_{\theta\,t}, t \geq 0\}$ that takes values in $\mathcal X = \mathcal H \times \mathcal F \times \mathcal D$ and is given by $X_{\theta\,0} = x = (h,f,d)' \in \mathcal H \times \mathcal F \times \mathcal D$
 and
\begin{align}\label{eqn:Xt}
X_{\theta\,t}
= \begin{bmatrix}
H_t \\ f_{\theta\,t} \\ d_{\theta\,t} 
\end{bmatrix}
= \begin{bmatrix}
g_{h1}(H_{t-1}) + g_{h2}(H_{t-1}) Z_{1\,t} \\
\sum_{k=1}^K \{\alpha_{0\,k} +  \alpha_{1\,k} \left[g_{y1}(H_{t-1}) + g_{y2}(H_{t-1})Z_{2\,t}\right]  + \beta_{1\,k }f_{\theta\,t-1} \}\mathbbm 1_{t-1\,k} \\
 1+ |g_{y1}(H_{t-1}) + g_{y2}(H_{t-1})Z_{2\,t}| + |f_{\theta\,t-1}|  + \overline \beta_1 d_{\theta\,t-1}
\end{bmatrix} 
\end{align}
for $t \geq 1$, where $\mathbbm 1_{t-1\,k} = \mathbbm 1_{\{g_{y1}(H_{t-1}) + g_{y2}(H_{t-1})Z_{2\,t} \in \mathcal Y_k\}} $, $Z_{1\,t} = \epsilon_{H\,t}$ and $Z_{2\,t} = \epsilon_{Y\,t-1}$.
We are interested in establishing that the companion Markov chain $\{ X_{\theta\, t}, t \geq 0 \}$ is $V_X$-geometrically ergodic \citep{Meyn:Tweedie:1993,Meitz:Saikkonen:2008}. 

\begin{defn}[$V_X$-geometric ergodicity]\label{defn:VX}
	A Markov chain $\{X_t, t \geq 0\}$ is $V_X$-geometrically ergodic if there exists a real valued function $V_X: \mathcal X \rightarrow [1,\infty)$, a probability measure $\pi_X$ on $\mathcal B (\mathcal X)$, and constants $\rho<1$ and $M_x < \infty$ (depending on $x$) such that
	\begin{equation}\label{eqn:VX_def}
	\sup_{v:|v|\leq V_X} \left| \int_{\mathcal X} P^n_X(x,dx_n)v(x_n) - \int_{\mathcal X} \pi_X(dx_n)v(x_n)\right| \leq \rho^n M_x ~,
	\end{equation}
	for all $x \in \mathcal X$ and all $n \geq 1$. 
\end{defn}

A number of remarks are in order.
First, the definition implicitly assumes that the expectation of the function $V_X$ with respect to the measure $\pi_X$ exists.
Second, a Markov chain that is $V_X$-geometric ergodic has convenient moment and dependence properties.
If we choose $V_X=1$ then we have that \eqref{eqn:VX_def} coincides with the definition of geometric ergodicity, which, in turn, allows to establish $\beta$- and $\alpha$-mixing.
Moreover, $V_X$-geometric ergodicity implies that the unconditional expectation of $v(X)$ exists for any function $v$ such that $|v| \leq V_X$.

The following lemma establishes that the companion Markov chain $\{X_{\theta\,t},t\geq 0\}$ is $V_X$-geometrically ergodic.

\begin{lemma_main}\label{lem:geometric_erg:x}
	Suppose Assumptions \ref{asm:dgp} and \ref{asm:algo} are satisfied.
	Then $\{X_{\theta\,t}, t\geq 0\}$ is $V_X$-geometrically ergodic with $V_X(x) = 1 + \|x\|_1^{2r_m}$.
\end{lemma_main}


The proof of Lemma \ref{lem:geometric_erg:x} is based on establishing that the Markov chain is irreducible, aperiodic and satisfies the so-called drift criterion.
The claim then follows from Theorem 15.0.1 of \cite{Meyn:Tweedie:1993}, which is a classic result that is routinely employed to establish stability of nonlinear time series models.

The following lemma establishes that the constants $\rho$ and $M_x$ in Definition \ref{defn:VX} in the case of geometric ergodicity (that is, when $V_X=1$) 
can be chosen so that they do not depend on $\theta$.\footnote{We omit the subscript $\theta$ from $x$ to simplify the notation, but the dependence on $\theta$ is understood.}

\begin{lemma_main}\label{lem:geometric_erg:rate}
	Suppose Assumptions \ref{asm:dgp} and \ref{asm:algo} are satisfied.
	Then, there exist positive constants $\rho \in (0,1)$ and $R < \infty$ that do not depend on $\theta$ such that $\{X_{\theta\,t},t\geq0\}$ satisfies
	\begin{equation}\label{eqn:GE_def}
	\sup_{v:|v|\leq 1} \left| \int_{\mathcal X} P^n_X(x,dx_{n})v(x_{n}) - \int_{\mathcal X} \pi_X(dx_{n})v(x_{n})\right| \leq R \tilde V_X(x) \rho^n  ~,
	\end{equation}
	for all $x \in \mathcal X$ and all $n \geq 1$, and $\tilde V_X(x) = 1 + \|x\|_1$.
\end{lemma_main}

The proof of Lemma \ref{lem:geometric_erg:rate} consists in an application of Theorem 12 of \citet{Roberts:Rosenthal:2004}.
We remark that the MCMC literature has developed a number of results that allow to establish explicit geometric ergodicity convergence rates \citep{Rosenthal:1995}.
The important implication of Lemma \ref{lem:geometric_erg:rate} is that the dependence properties of the companion Markov chain $\{ X_{\theta \,t}, t\geq 0\}$ can be characterized independently of $\theta$.

The second step of the analysis consists is using the properties of the companion Markov chain $\{X_{\theta\,t},t\geq0\}$ to establish the properties of the joint process $\{(Y_t,X_{\theta\,t})',t\geq0\}$.
The following lemma establishes the connection between the transition kernels of $\{X_{\theta\,t},t\geq0\}$ and $\{(Y_t,X_{\theta\,t})',t\geq0\}$.

\begin{lemma_main}\label{lem:joint_transition}
	Consider the Markov chain $\{(Y_t,X_{\theta\,t})',t\geq0\}$ defined above.
	Let $\pi_{Y|X}(dy|x_t)$ denote the (invariant) conditional distribution of $Y_t$ given $X_{\theta\,t}=x_t$. 
	Then, its $n$-step transition kernel is given by
	\begin{align}\label{eqn:joint_transition}
	P^n_{Y,X}((y,x),d(y_n,x_n)) = 
	\pi_{Y | X}(dy_n|x_n)\int_{\mathcal H}P_X^{n-1}(\tilde x, dx_n)P_H(h,dh_1),\quad n\geq 2,
	\end{align}
	where $P_H$ is the transition kernel of $\{H_t,t\geq0\}$, and\\
	$\tilde x = \tilde x(y,x,h_1)=
	( h_1, 
	\sum_{k=1}^K (\alpha_{0\,k}+\alpha_{1\,k}y + \beta_{1\,k}f )\mathbbm 1_{\{y \in \mathcal Y_k\}} ,
	1 + |y| + |f| + \overline \beta_1 d )'$.
\end{lemma_main}
The proof of the lemma builds upon the analysis of GARCH models of \cite{Meitz:Saikkonen:2008}.
The structure given by equations 
\eqref{eqn:dgp:1}, \eqref{eqn:dgp:2}, \eqref{eqn:forecast:threshold} and \eqref{eqn:bounding_process} allows us to cast $\{(Y_t,X_{\theta\,t})',t\geq0\}$ as a Markov chain with Dirac measure as the initial distribution.
We remark that the analysis of $\{X_{\theta\,t},t\geq0\}$ differs depending on whether the process is studied in isolation or jointly with the process $\{Y_t,t\geq 0\}$.
The random vector $X_{\theta\,t}$ depends on $Y_{t-1}$.
When the process $\{X_{\theta\,t},t\geq 0\}$ is analyzed in the joint system $\{(Y_t,X_{\theta\,t})',t\geq0\}$ we have that the 1-step-ahead transition kernel of the process conditions on $Y_{t-1}$.
However, when $\{X_{\theta\,t},t\geq 0\}$ is analyzed in isolation we have that the 1-step-ahead transition kernel of the process does not condition on $Y_{t-1}$.


The following lemma establishes that $\{(Y_t,X_{\theta \,t})',t\geq0\}$ inherits the moment and dependence properties of the companion Markov chain $\{X_{\theta \,t}, t\geq0\}$.

\begin{lemma_main}\label{lem:geometric_erg:yx}
	Suppose Assumptions \ref{asm:dgp} and \ref{asm:algo} are satisfied.
	Then $(i)$ $\{(Y_t,X_{\theta \,t})', t \geq0 \}$ is $V_{Y,X}$-geometrically ergodic with $V_{Y,X}(y,x) = 1 + |y|^{2r_m} + \|x\|^{2r_m}_1$; and
	$(ii)$ there exist positive constants $\rho \in (0,1)$ and $R < \infty$ that do not depend on $\theta$ such that 
	 $\{(Y_t,X_{\theta \,t})',t\geq0\}$ satisfies
	\begin{eqnarray*}
		\sup_{v:|v| \leq 1} 
		\left\lvert 
		\int_{\mathcal Y \times \mathcal X} [P^n_{Y,X}((y,x),d(y_n,x_n))-\pi_{Y,X}(d(y_n,x_n))]v(y_n,x_n)
		\right\rvert
		\leq 
		R\tilde V_{X}(\check x)\rho^{n},
	\end{eqnarray*}
for all $(y,x)' \in \mathcal Y \times \mathcal X$ and for all $n \geq 2$, and $ \check x = (h , \overline \alpha_0 + \overline \alpha_1 |y| + \overline \beta_1 |f| , 1 + |y|+|f| +\overline\beta_1 d )'$.
\end{lemma_main}

Finally, we establish the moment and dependence properties of $\{ (Y_t,X_{\theta \,t})', t\geq 0\}$.
We introduce some further notation.
We define the $L_r$ norm of a random variable $X$ as $\|X\|_{L_r} = \left(\mathbb E|X|^r\right)^{1/r}$ for any $r \in [1, \infty)$.
We define the $\alpha$-mixing coefficients of the process $\{ (Y_t, X_{\theta \,t}')',t \geq 0\}$ as 
\begin{equation*}
\alpha(l) = \sup_{A \in \mathcal{F}_{-\infty}^s, B \in \mathcal{F}_{s+l}^{\infty}} \left| {\mathbb P \left(A \cap B \right) - \mathbb P \left(A\right) \mathbb P \left(B\right) } \right| ~,
\end{equation*}
where $\mathcal{F}_{-\infty}^s$ and $\mathcal{F}_{s+l}^{\infty}$ denote the $\sigma$-algebras generated by $\lbrace (Y_t, X_{\theta \,t}')': 0 \leq t \leq s\rbrace$ and $\lbrace (Y_t, X_{\theta \,t}')': s + l \leq t \leq \infty\rbrace$ respectively.

\begin{prop}\label{prop:mom_and_mixing}
	Suppose Assumptions \ref{asm:dgp} and \ref{asm:algo} are satisfied. Then, the process $\{(Y_t,X_{\theta\,t})',t\geq0\}$
	$(i)$ satisfies $\| Y_t \|_{L_{2r_m}} < \infty$, $\| H_t \|_{L_{2r_m}} < \infty$, $\sup_{\theta \in \Theta} \| f_{\theta\,t} \|_{L_{2r_m}} < \infty$ and $\sup_{\theta \in \Theta } \| d_{\theta\,t} \|_{L_{2r_m}} < \infty$; and
	$(ii)$ has $\alpha$-mixing coefficients that satisfy $\alpha(l) \leq \exp\left( -C_\alpha l^{r_\alpha} \right)$ for some $C_\alpha>0$ and $r_\alpha>0$ that do not depend on $\theta$.
\end{prop} 

\subsection{Establishing Performance Bounds for the ERM}

We introduce a general inequality to bound the performance of the ERM.

\begin{prop}\label{prop:generalized_gabor}	
	Let $\overline R(\theta) = \mathbb E R(\theta) $. 
	Then, it holds that 
	\begin{align}\label{eqn:general_gabor}
	R( \hat \theta) - \inf_{\theta \in \Theta } R(\theta)
	\leq
	2\sup_{\theta\in\Theta}|R(\theta)-\overline R( \theta )|+
	2\sup_{\theta\in\Theta}| R_T(\theta) - \overline{R}( \theta )|~.
	\end{align}
\end{prop}

It is important to emphasize that Proposition \ref{prop:generalized_gabor} is a general result that only requires the loss process to be a stationary sequence.
We note that when the data is i.i.d.~we have that $R(\theta) =\overline R(\theta)$ and the inequality in Proposition \ref{prop:generalized_gabor} corresponds to
the classic inequality derived in \citet{Vapnik:Chervonenkis:1974} \citep{Devroye:Gyorfi:Lugosi:1996}, which is routinely used to derive bounds on the performance of the ERM.

The first term of the inequality in \eqref{eqn:general_gabor} is the supremum of a difference between an average of conditional and unconditional expectations.
Proposition \ref{prop:cond-uncond} bounds this term using Ibragimov's inequality \citep[Theorem 14.2]{Davidson:1994}.

\begin{prop}\label{prop:cond-uncond} 
	Suppose Assumptions \ref{asm:dgp} and \ref{asm:algo} are satisfied. Then, for all $T$ and for $C = {12 \over \gamma} \sup_{\theta\in\Theta}\|L(Y_t,f_{\theta\,t})\|_{L_{r_m}} \sum_{l=1}^{\infty} e^{-\left((r_m-1)/r_m\right)C_\alpha l^{r_\alpha}}$, it holds that
	\begin{align}\label{eqn:cond-uncond}
	\mathbb P\left( \sup_{\theta\in\Theta}|R(\theta)-\overline R( \theta )| > {C \over 2 \sqrt{T} } \right) \leq {1 \over \sqrt{T} } ~.
	\end{align}
\end{prop}

The second term of the inequality in \eqref{eqn:general_gabor} is the supremum of the empirical process associated with the prediction loss process.
We bound this using a covering argument and a concentration inequality for $\alpha$-mixing processes. 
More precisely, Proposition \ref{prop:covering} is based on a covering argument similar to \citet{Jiang:Tanner:2010}. 
Importantly, the proof of Proposition \ref{prop:covering} relies on the dominating process $\{d_{\theta\,t},t\geq 0\}$ and, in particular, on the property spelled out in \eqref{eqn:domination}.
Proposition \ref{prop:conc} is based on a Bernstein-type inequality for $\alpha$-mixing sequences from \citet{Liebscher:1996}.

\begin{prop}\label{prop:covering}
Suppose Assumptions \ref{asm:dgp} and \ref{asm:algo} are satisfied. 
Then, for any $\varepsilon \in (0,16 C_d]$ it holds that
\begin{eqnarray*}
	&& \mathbb P\left( \sup_{\theta \in \Theta} |R_T(\theta) - \overline R(\theta)| > {\varepsilon \over 2 }\right) \\
	&& \quad \leq \left( 1 + {32C_\Theta C_d \over \varepsilon} \right)^p \sup_{ \theta \in \Theta }   \left[ \mathbb P\left( \left|{1\over T}\sum_{t=1}^T \widetilde U_{\theta\,t}  \right| > {\varepsilon\over 4} \right) 
	+ \mathbb P\left( \left| {1\over T} \sum_{t=1}^T \widetilde V_{\theta\,t} \right| > 2 C_d \right) \right] ~,
\end{eqnarray*}
where $C_\Theta = \sup_{\theta \in \Theta} \| \theta \|_2$,
$C_d= C_\psi  \sup_{\theta} \| d_{\theta\,t}^2 + 2 |Y_t-f_{\theta\,t}| d_{\theta\,t} \|_{L_1}$, 
$\widetilde U_{\theta\,t}= U_{\theta\,t} - \mathbb E U_{\theta\,t}$,
$\widetilde V_{\theta\,t}= V_{\theta\,t} - \mathbb E V_{\theta\,t}$,
$ U_{\theta\,t}   =  L(Y_t , f_{\theta\,t}) $ and
$ V_{\theta\,t}  =  C_\psi \left( d_{\theta\,t}^2 + 2 |Y_t-f_{\theta\,t}| d_{\theta\,t}\right) $.
\end{prop}

\begin{prop}\label{prop:conc}

	Suppose Assumptions \ref{asm:dgp} and \ref{asm:algo} are satisfied. 
	Then, for all $T$ sufficiently large and for \( \varepsilon_T =  46 \sigma \sqrt{ p \log T / T } \), it holds that 
	\begin{align*}
		& \left( 1 + {32 C_\Theta C_d \over \varepsilon_T } \right)^p \sup_{ \theta \in \Theta } \mathbb P\left( \left|{1\over T} \sum_{t=1}^T \widetilde U_{\theta\,t} \right| > { \varepsilon_T \over 4} \right) 
		\leq {1 \over \log T } \text{ and }\\
		& \left( 1 + {32 C_\Theta C_d \over \varepsilon_T } \right)^p \sup_{ \theta \in \Theta } \mathbb P\left( \left|{1\over T} \sum_{t=1}^T \widetilde V_{\theta\,t} \right| > 2 C_d \right) 
		\leq o\left({1 \over \log T }\right) \text{ as } T \rightarrow \infty ~,
	\end{align*}
	where $\sigma^2 = 16 {r_m \over r_m -2}\sup_{\theta \in \Theta} \| L(Y_t,f_{\theta\,t}) - \mathbb E L(Y_t,f_{\theta\,t}) \|_{L_{r_m}}^2 ( 1 + 2 \sum_{l=1}^\infty \exp\left(-C_\alpha l^{r_\alpha}\right)^{1-{2\over r_m} })$.
\end{prop}

It follows from Proposition \ref{prop:cond-uncond}, \ref{prop:covering} and \ref{prop:conc} that, for $T$ sufficiently large, 
\[
	2 \sup_{\theta\in\Theta}|R(\theta)-\overline R( \theta )| + 2 \sup_{\theta \in \Theta} |R_T(\theta) - \overline R(\theta)| \leq 47 \sigma \sqrt{ p \log T \over T } 
\] 
holds with high probability.
This fact and Proposition \ref{prop:generalized_gabor} imply Theorem \ref{thm:erm}.

\section{Conclusions}\label{sec:end}

Leo Breiman forcefully argued that there are two main philosophies to analyze data \citep{Breiman:2001}, the data modeling and the algorithmic modeling cultures.
The data modeling culture is based on assuming that the data is generated by a (partially) known model whereas 
the algorithmic modeling culture pursues to be agnostic about the data generating mechanism.
It is fair to say that the majority of research in the time series literature is typically carried out trough the lens of the data modeling culture,
whereas the fraction of contributions from the algorithmic modeling perspective is meager.

In this work we take the algorithmic standpoint and study the performance of empirical risk minimization to choose an algorithm to forecast 1-step-ahead a time series.
A key feature of the analysis is that the relationship between the time series and the class of algorithms is unspecified.
Our main result implies that the algorithm chosen by empirical risk minimization achieves asymptotically the optimal predictive performance that is attainable within the class.

The algorithmic modeling culture paves the way for the development of new forecasting strategies for time series applications.
Using the tools introduced in the nonlinear time series literature it is possible to develop general nonparametric theory to study the properties of these algorithmic forecasting strategies from primitive assumptions.

\appendix

\section{Proofs of Section \ref{sec:erm} and \ref{sec:proof}}

To simplify the analysis and without loss of generality we assume that $\Theta = [\underline \alpha_0,\overline \alpha_0]^K \times [\underline \alpha_1,\overline \alpha_1]^K \times [0,\overline \beta_1]^K$.
To simplify notation we write $W_{\theta\,t}=(Y_t,X_{\theta\,t})'$. 

\begin{proof}[Proof of Theorem \ref{thm:erm}]
The claim follows from Propositions \ref{prop:mom_and_mixing}, \ref{prop:generalized_gabor}, \ref{prop:cond-uncond}, \ref{prop:covering} and \ref{prop:conc}.
\end{proof}

\begin{proof}[Proof of Lemma \ref{lem:geometric_erg:x}]
We apply Lemmas \ref{lem:irre}, \ref{lem:aperiod} and \ref{lem:drift} together with Theorem 15.0.1 of \cite{Meyn:Tweedie:1993} to obtain that $\{X_{\theta\,t},t\geq0\}$ is $q_X$-geometrically 
ergodic with $q_X(x)=1+(\kappa'\dot x)^{2r_m}$, where $\dot x = (|h|,|f|,|d|)'$, and the vector $\kappa \in (0,1)^3$ is defined in Lemma \ref{lem:irre}.
Moreover, it is easy to see that Lemma \ref{lem:drift} still holds with $q_X(x)$ replaced with $q_X(x)/\underline \kappa^{2r_m}$, where $\underline \kappa$ is the minimum of the components of $\kappa$. 
The claim follows by noting that $V_X(x)=1+\|x\|_1^{2r_m} \leq q_X(x)/\underline \kappa^{2r_m}$.
\end{proof}

\begin{proof}[Proof of Lemma \ref{lem:geometric_erg:rate}]
The claim of the Lemma follows from an application of Theorem 12 by \cite{Roberts:Rosenthal:2004}.
Define $\tilde q_X(x) = 1 + \tilde \kappa_h |h| + \tilde \kappa_f |f| + \tilde \kappa_d |d| = 1 + \tilde \kappa'\dot x$ where $\tilde \kappa \in (0,1)^3$
as well as the set $\tilde S_{2\,\epsilon}=\left\lbrace x\in\mathcal X: \tilde\kappa'\dot x\leq \tilde M\right\rbrace$.
By arguments analogous to those used to claim that $S_{2\,\epsilon}$ defined in Lemma \ref{lem:irre} is small we can show that we can choose a $\tilde \kappa$ such that 
for any $x \in \tilde S_{2\,\epsilon}$ and any $A \in \mathcal B (\mathcal X)$ it holds that $P_X^2( x , A ) \geq \tilde c_* \tilde\varphi(A) $,
where $\tilde c_* \in (0,1)$ and $\tilde \varphi(A)=\mu_{Leb}(A\cap \tilde D)$ is Lebesgue measure restricted to an open rectangular region $\tilde D$, which is the analogue of $D$ defined in Lemma \ref{lem:irre}.
As we remark in the proof of Lemma \ref{lem:irre} $\tilde c_*$ and $\tilde D$ do not depend on $\theta$.

It is easily verified that $\tilde q_X(x)$ satisfies the drift criterion by the same arguments as in Lemma \ref{lem:drift}, and that $\tilde q_X(x) \leq \tilde V_X(x)$.
Define $\lambda^{-1} 
=1 - \tilde \gamma_1 + 
\tilde \gamma_2 / (2 + \tilde M),$
where $\tilde \gamma_1$ and $\tilde \gamma_2$ are analogous to $\gamma_1$ and $\gamma_2$ in Lemma \ref{lem:drift} and do not depend on $\theta$, and $\tilde M = \inf_{x \in S_{2\,\epsilon}^c} \tilde q_X(x)$.
The proof strategy of Theorem 12 by \cite{Roberts:Rosenthal:2004} is based on a coupling argument. 
To this end we use $\{X_{\theta\,t}^G,t\geq0\}$ to denote an independent copy of the Markov chain $\{X_{\theta\,t},t\geq0\}$ started at the stationary distribution, namely $X_{\theta\,0}^G\sim \pi_X$.
We define 
$
B = \max \lbrace 1 , \lambda^2 (1 - \tilde c_*)  \overline R \rbrace,
$
where the constant $\overline R$ is computed in Lemma \ref{lem:integral_rosenthal} from the Online Appendix.
We distinguish two cases. Note that in both cases we are applying Lemma \ref{lem:integral_rosenthal} from the Online Appendix.\\
$(i)$
Suppose that $\lambda^{-1} < 1$.
Then the assumptions of Proposition 11 and Theorem 12 by \cite{Roberts:Rosenthal:2004} are satisfied, thus applying the theorem we have that for any $ j \in \{ 1, \ldots, n \}$,
\begin{equation*}\label{eqn:Thm12_RobertsRosenthal}
\sup_{v:|v|\leq 1} \left| \int_{\mathcal X} \left[P^n_X(x,dx_1) - \pi_X(dx_1)\right]v(x_1)\right| \leq 
	(1-\tilde c_*)^j + \lambda^{-n} { B^{j-1} \over 2} \left( \tilde q_X(x)  + \mathbb E \tilde q_X(X_{\theta\,t})  \right)
\end{equation*}
holds for all $x \in \mathcal X$ and all $n \geq 1$.
Let 
$
\overline  V = 1 + \|H_t\|_{L_1} + \sup_{\theta \in \Theta} \|f_{\theta\,t}\|_{L_1}+
\sup_{\theta \in \Theta } \|d_{\theta\,t}\|_{L_1}$.
Obviously, $\overline V \geq \mathbb E \tilde q_X(X_{\theta\,t})$.
Furthermore, $\overline V < \infty$ by Lemma \ref{lem:geometric_erg:x} and Proposition \ref{prop:mom_and_mixing}$(i)$.
Set $j = \lfloor rn \rfloor$ for sufficiently small $r > 0$ so that the bound converges to zero at a geometric rate. 
We now have that \eqref{eqn:GE_def} holds with $\rho = (1 - \tilde c_*)^r \vee \left(\lambda^{-1} B^{r}\right)< 1$ and $R = 2 \overline V$ (note that $\tilde q_X \geq 1$). The result follows since $\rho$ and $R$ do not depend on $\theta$. \\
$(ii)$ In the case $\lambda^{-1} \geq 1$, we can find an enlargement of $\tilde S_{2\,\epsilon}$ for which the result in $(i)$ still holds  \citep{Roberts:Rosenthal:2004}.
We choose $M'$ such that $2 + M' > \tilde \gamma_2 / \tilde \gamma_1$.
Note that 
$
\tilde S'_{2\,\epsilon} = \{x\in\mathcal X: \tilde \kappa'\dot x \leq M' \}
$
is still a small set by the same arguments used in the proof of Part II of Lemma \ref{lem:irre}. Consequently, $P_X^2(x,A) \geq \tilde c_*' \tilde\varphi'(A)$ for all $x \in \tilde S_{2\,\epsilon}'$, where $\tilde c_*'$ is possibly smaller than $\tilde c_*$ but strictly positive (and independent of $\theta$), $\tilde \varphi'(A)=\mu_{Leb}(A\cap\tilde D')$, and $\tilde D'$ is analogous to $\tilde D$ in part $(i)$.
Clearly, $\lambda'^{-1} = 1 - \tilde \gamma_1 + {\tilde \gamma_2 \over 2 + M'} < 1$. 
The result now follows by using the same arguments as in part $(i)$ with $\lambda$ and $\tilde c_*$ replaced by $\lambda'$ and $\tilde c_*'$.
\end{proof}

\begin{proof}[Proof of Lemma \ref{lem:joint_transition}]
	For all $n\geq2$ we write 
	\begin{align*}
	P^n_W(w,dw_n) 
	&= \pi_{Y|X}(dy_n|x_n)\mathbb P(dx_n|w)
	=\pi_{Y|X}(dy_n|x_n)\int_{\mathcal H} \mathbb P(dx_n|w,h_1) P_H(h,dh_1)~,
	\end{align*}
	where the last equality follows because the $H_t$ component of $W_{\theta\,t}$ is a Markov chain of its own.
	Define $\tilde f_{\theta} = f_{\theta\,1}=\sum_k \{\alpha_{0k}+\alpha_{1k}y + \beta_1 f\}\mathbbm 1_{\{y \in \mathcal Y_k\}}$, and $\tilde d = d_1 = 1 + |y| + |f| + \overline \beta_1 d$. 
	Note that
	by the i.i.d.~assumption on the innovations $Z_{1\,t}$ and $Z_{2\,t}$ we have that the $X_{\theta\,t}$ component of $\{W_{\theta\,t},t\geq0\}$ has a 2-step transition mechanism which is entirely similar to the 1-step transition mechanism of the companion Markov chain defined in \eqref{eqn:Xt} with initial value given by $\tilde x(w,h_1)$.
	We denote $\tilde P_X^n((w,h_1), dx_n) = \mathbb P(dx_n| w,h_1)$.
	Note that
	$
	\tilde P_X^2((w,h_1), dx_2) = P_X^1(\tilde x, dx_2)
	$
	where $\tilde x =\tilde x(w,h_1)= (h_1,\tilde f_\theta, \tilde d)$. 
	Next, 
	$
	\tilde P_X^3 ((w,h_1), dx_3)
	= \int_{\mathcal X} P_X(\tilde x, dx_2)P_X(x_2,dx_3) =P_X^2(\tilde x, dx_3).
	$
	By induction, $\tilde P_X^n((w,h_1), dx_n)=P_X^{n-1}(\tilde x, dx_n)$, and the result follows.
\end{proof}

\begin{proof}[Proof of Lemma \ref{lem:geometric_erg:yx}]

$(i)$ First, $\{X_{\theta\,t}, t \geq 0\}$ viewed as a separate Markov chain is $V_X$-geometrically ergodic by Lemma \ref{lem:geometric_erg:x}. 
We begin by showing that
\begin{align}\label{eqn:inheritance_cond}
\mathbb E_{Y|X}\left( V_{Y,X}(W_{\theta\,t}) \right) \equiv 
\int_{\mathcal Y} V_{Y,X}(y,x)\pi_{Y | X}(dy|x) < C \cdot V_X(x)~.
\end{align} 
For all $x \in \mathcal X$, by Assumption \ref{asm:dgp}$(ii)$ we have that
\begin{align*}
	\mathbb E_{Y|X}\left(V_{Y,X}(W_{\theta\,t})\right)
	&= V_X(x) + \mathbb E_{Y|X}\left(|g_{y1}(h)+g_{y2}(h)\epsilon_{Y\,t}|^{2r_m}\right) \\
	&\leq
	V_X(x) + 2^{2r_m-1}|g_{y1}(h)|^{2r_m} + 2^{2r_m-1}|g_{y2}(h)|^{2r_m} 
	\mathbb E_{Y|X}(|\epsilon_{Y\,t}|^{2r_m}) \\
	&\leq V_X(x) + C \|x\|_1^{2r_m}
	\leq C \cdot V_X(x)~,
\end{align*}
where the constant $0<C<\infty$ may change from line to line.
To satisfy the definition of $V_{Y,X}$-geometric ergodicity, we must have that $\mathbb E \left(V_{Y,X}(Y_t,X_{\theta\,t})\right) < \infty$, where the expectation is taken with respect to the invariant measure $\pi_{Y,X}$. 
By \eqref{eqn:inheritance_cond} we have that
\begin{align*}
\mathbb E \left(V_{Y,X}(W_{\theta\,t})\right)
= \int_{\mathcal X}\pi_X(dx)\int_{\mathcal Y}V_{Y,X}(y,x)\pi_{Y|X}(dy|x)
\leq \int_\mathcal X \pi_X(dx) C\cdot V_X(x)<\infty,
\end{align*}
as expected.
For any  $w=(y,x)' \in \mathcal Y \times \mathcal X$ and all $n\geq 2$ we have that
\begin{align}\label{eqn:geom_inheritance}
&\sup_{v:|v| \leq V_{Y,X}}
\left\lvert
\int_{\mathcal Y\times \mathcal X} \left[P_{Y,X}^n(w,dw_n) -\pi_{Y,X}(dw_n)\right]v(w_n)
\right\rvert\nonumber\\
&=
\sup_{v:|v| \leq V_{Y,X}}
\left\lvert
\int_{\mathcal X} \int_{\mathcal H} P_X^{n-1}(\tilde x,dx_n)P_H(h,dh_1)-\pi_X(dx_n)\int_{\mathcal Y} \pi_{Y|X}(dy_n|x_n)v(y_n,x_n)
\right\rvert\nonumber\\
&\quad\leq
C\sup_{v':|v'| \leq V_X}
\left\lvert
\int_{\mathcal H} \left(\int_{\mathcal X} \left[P_X^{n-1}(\tilde x,dx_n)-\pi_X(dx_n)\right]v'(x_n)\right)P_H(h,dh_1)
\right\rvert\nonumber\\
&\quad\leq
C\int_{\mathcal H} 
\sup_{v':|v'| \leq V_X}
\left\lvert \int_{\mathcal X} \left[P_X^{n-1}(\tilde x,dx_n)-\pi_X(dx_n)\right]v'(x_n)\right\rvert P_H(h,dh_1)\nonumber \\
&\quad\leq C R_\theta\rho_\theta^{n-1} \mathbb E\left(V_X(\tilde x)\vert H_0=h\right)~,
\end{align}
where $R_\theta < \infty$, $\rho_\theta<1$.
The equality follows by Lemma \ref{lem:joint_transition},
the first inequality is a consequence of \eqref{eqn:inheritance_cond} and the last inequality is implied by the drift criterion that we have used in the proof of Lemma \ref{lem:geometric_erg:x}.
Furthermore, note that by Assumptions \ref{asm:dgp}$(i)$, $(iii)$, $(iv)$ and \eqref{eqn:hbounds} we have
\begin{align*}
\mathbb E\left(V_X(\tilde x)\vert H_0=h\right)
&\leq 1 + 3^{2r_m-1}\left(  \mathbb E\left(|H_1|^{2r_m}\vert H_0=h\right)+|\tilde f_\theta|^{2r_m} + |\tilde d|^{2r_m}  \right)\\
&<
1+C\cdot\begin{cases}
|\tilde f_\theta|^{2r_m} + |\tilde d|^{2r_m}+|h|^{2r_m}~, & |h|>M_\epsilon\\
1+|\tilde f_\theta|^{2r_m} + |\tilde d|^{2r_m}
~, & |h|\leq M_\epsilon
\end{cases}\\
&< C V_X(\check x)~,
\end{align*}
where $1<C<\infty$ may change from line to line and the choice of $\epsilon$ is such that $\mathbb E(a_h+b_h|\epsilon_{H\,t}|+\epsilon)^{2r_m}<1$ (Assumption \ref{asm:dgp}$(iv)$).\\
$(ii)$ Repeating the same arguments as in $(i)$ with $2r_m = 1$ and with $\sup_{v:|v|\leq1}$ instead of $\sup_{v:|v|\leq V_{Y,X}}$, we can use Lemma \ref{lem:geometric_erg:rate} in the last inequality of \eqref{eqn:geom_inheritance} instead of the standard drift criterion to obtain constants $\rho \in (0,1)$ and $R<\infty$ that do not depend on $\theta$. The proof is completed by noting that we can redefine $R$ to absorb $C\rho^{-1}$.
\end{proof}

\begin{proof}[Proof of Proposition \ref{prop:mom_and_mixing}]
$(i)$ 
We begin by noting that Lemma \ref{lem:geometric_erg:yx} implies that $\|Y_t\|_{L_{2r_m}}$, $\|H_t\|_{L_{2r_m}}$, $\|f_{\theta\,t}\|_{L_{2r_m}}$ and $\|d_{\theta\,t}\|_{L_{2r_m}}$ exist.
To complete the proof we need to establish that  $\sup_{\theta \in \Theta} \|f_{\theta\,t}\|_{L_{2r_m}}$ and $\sup_{ \theta \in \Theta  } \| d_{\theta\,t}\|_{L_{2r_m}}$ are finite.
Since $\alpha_{1\,k}>0$, we have that
$
\|f_{\theta\,t}\|_{L_{2r_m}} 
\leq |\overline \alpha_0| + \overline \alpha_1 \|Y_t\|_{L_{2r_m}} + \overline \beta_1 \|f_{\theta\,t}\|_{L_{2r_m}}$
by the existence of the stationary distribution of $\{W_{\theta\,t},t\geq0\}$ which follows from Lemma \ref{lem:geometric_erg:yx}.
The last inequality and Assumption \ref{asm:algo} imply that 
$
	\sup_{\theta \in \Theta} \|f_{\theta\,t}\|_{L_{2r_m}} \leq  (1 - \overline \beta_1 )^{-1} \left(|\overline \alpha_0| + \overline \alpha_1 \|Y_t\|_{L_{2r_m}}\right) <\infty ~.
$
Analogously, we have that 
$
\|d_{\theta\,t}\|_{L_{2r_m}} 
\leq 1 + \|Y_{t}\|_{L_{2r_m}} + \| f_{\theta\,t} \|_{L_{2r_m}}  + \overline \beta_1 \| d_{\theta\,t} \|_{L_{2r_m}}.
$
Thus,
$\sup_{\theta \in \Theta  } \|d_{\theta\,t}\|_{L_{2r_m}} \leq (1 - \overline \beta_1)^{-1} \left( 1 + \| Y_t \|_{L_{2r_m}} + \sup_{\theta \in \Theta} \|f_{\theta\,t}\|_{L_{2r_m}} \right) < \infty.$ \\
$(ii)$
It is enough to show that $\{W_{\theta\,t},t\geq0\}$ is geometrically $\beta$-mixing, since $\alpha(l) \leq \beta(l)$, where
$
\beta(l) = 
\sup_{t\in\mathbb Z}
{1\over 2}\sup \sum_{i=1}^I\sum_{j=1}^J
\left\lvert \mathbb P( A_i \cap B_i ) - \mathbb P(A_i)\mathbb P(B_i)\right\rvert,
$
and the supremum is taken over all pairs of finite partitions
$\{A_1,\dots,A_I \}$ and $\{B_1,\dots,B_J \}$ of $\Omega$ such that
$A_i \in \sigma\{W_{\theta\,s}: s \leq t\}$, $i=1,\dots,I$, and 
$B_j \in \sigma\{W_{\theta\,s}: s \geq t + l\}$, $j=1,\dots,J$. 
Let $\delta_{w}(A) = \mathbbm 1\{w\in A\}$ for any $A \in \mathcal B(\mathcal Y \times \mathcal X)$.
By Proposition 4 in \cite{Liebscher:2005}, $\{W_{\theta\,t},t\geq0\}$ is $\beta$-mixing with geometrically decaying mixing numbers if 
$(a)$ $\int_{\mathcal Y \times \mathcal X} V_X(x_0) \delta_{y,x}(dw_0)\} =V_X(x) < \infty$, and
$(b)$ $\{W_{\theta\,t},t\geq0\}$ is $Q$-geometrically ergodic in the sense of \cite{Liebscher:2005}
with $Q(w) = V_X(x)$. Condition $(a)$ holds for all $w \in \mathcal Y \times \mathcal X$.
For condition $(b)$, we first need to show that 
$\int_{\mathcal Y \times \mathcal X} V_{X}(x_n) \pi_{Y,X}(dw_n) < \infty$. 
This is easily obtained by noting that
\begin{eqnarray*}
\int_{\mathcal Y \times \mathcal X} V_{X}(x_n) \pi_{Y,X}(dw_n) 
=\int_{\mathcal X} V_{X}(x_n)\pi_X(dx_n) \int_{\mathcal Y} \pi_{Y|X}(dy_n|x_n)
< \infty,
\end{eqnarray*}
where the last inequality follows from the $V_X$-geometric ergodicity of $\{X_{\theta\,t},t\geq0\}$.
As for the remaining part of condition $(b)$, notice that from Lemma \ref{lem:geometric_erg:yx}$(ii)$ we have that
$\left\|P_{Y,X}^l(w,\cdot) - \pi_{Y,X}\right\|_{TV}
	\leq R \tilde V_X(\check x)\rho^l \wedge 1$,
where 
\[
\left\|P_{Y,X}^l(w,\cdot) - \pi_{Y,X}\right\|_{TV} = 
\sup_{v:|v|\leq 1}\left\lvert \int_{\mathcal Y \times \mathcal X} \left[P^l_{Y,X}\left(w , dw_l\right) - \pi_{Y,X}(dw_l)\right] v(w_l)\right\rvert~,
\]
which completes the proof of condition $(b)$.
It remains to be shown that the rate of decay does not depend on $\theta$.
For any probability measure $\tau$ on $\mathcal Y\times \mathcal X$, define
$
\xi_l(\tau)=\int_{\mathcal Y \times \mathcal X} 
\left\|P_{Y,X}^l(w,\cdot) - \pi_{Y,X}\right\|_{TV} \cdot  \tau(dw).
$
By virtue of part $(ii)$ of Lemma \ref{lem:geometric_erg:yx} we compute that 
$\xi_l(\pi_{Y,X}) \leq R\check V \rho^l$, where 
\begin{align*}
\check V= \left(2 + \|H_t\|_{L_1} + (1+\overline\alpha_1)\|Y_t\|_{L_1}+(1+\overline\beta_1) \sup_{\theta\in\Theta}\|f_{\theta\,t}\|_{L_1} + \overline \beta_1 \sup_{\theta\in\Theta}\|d_{\theta\,t}\|_{L_1}\right)~,
\end{align*}
and $\xi_l(\delta_{y,x}) = \left\|P_{Y,X}^l(w,\cdot) - \pi_{Y,X}\right\|_{TV} \leq R\tilde V_X(\check x)\rho^l \wedge 1.$
Now, by Proposition 3 in \cite{Liebscher:2005} we have that for any $w \in \mathcal Y\times \mathcal X$, and $m=\lfloor l/2\rfloor$,
$
\beta(l) \leq 3 \xi_m(\delta_{y,x}) + \xi_m(\pi_{Y,X}) \leq 
R\left(\check V + 3 \tilde V_X(\check x) \right) \rho^m \wedge 1.
$
It is not difficult to verify that
$
\alpha(l) \leq \beta(l) \leq \exp\left(-C_\alpha l^{r_\alpha}\right) \wedge 1
$
for all $l\geq1$.
The choice of $C_\alpha$ and $r_\alpha$ depends on $R,\check V,\rho$ and $\tilde V_X(\check x)$.
Since none of those depend on $\theta$ by Lemma \ref{lem:geometric_erg:rate}, neither does the rate of decay of the uniform bound for the $\alpha$-mixing coefficients. The claim follows by redefining $R$ and noting that $\tilde V_X \geq 1$.
\end{proof} 

\begin{proof}[Proof of Proposition \ref{prop:generalized_gabor}]
	Let $\{ (Y^G_t , f^G_{\theta\,t})' , t\geq 0 \}$ be an independent copy of the process $\{ (Y_t , f_{\theta\,t})' , t\geq 0 \}$
	and define $\overline R(\hat \theta) = \mathbb E\left[ {1\over M} \sum_{t=T+1}^{T+M} L(Y_t^G,f^G_{\hat \theta\,t}) \right]$.
	By the properties of infimum and supremum and the definition of empirical risk minimizer (i.e.~$R_T(\theta)\geq R_T(\hat\theta)$ for all $\theta \in \Theta$), we have that
	\begin{align*}
	R (\hat \theta ) - \inf_{\theta \in \Theta } R(\theta)
	&= 
	R( \hat \theta ) - \overline R(\hat\theta) + \overline R(\hat\theta) - \inf_{\theta\in\Theta} \left[ \overline R(\theta)+ R(\theta)-\overline R(\theta) \right] \nonumber \\
	&\leq
	\left[R( \hat \theta ) - \overline R(\hat\theta)\right] + \left[\overline R(\hat\theta) -\inf_{\theta\in\Theta} \overline R(\theta)\right]- \inf_{\theta\in\Theta} \left[ R(\theta)-\overline R(\theta) \right] \nonumber\\
	&\leq 2\sup_{\theta\in\Theta}|R(\theta)-\overline R(\theta)| + 2\sup_{\theta\in\Theta}|R_T(\theta)-\overline R(\theta)| ~, 
	\end{align*}
	where the last inequality follows from Lemma 8.2 in \cite{Devroye:Gyorfi:Lugosi:1996}.
\end{proof}

\begin{proof}[Proof of Proposition \ref{prop:cond-uncond}]
	
First, by the triangular inequality we have
\begin{align}\label{eqn:triangular}
\sup_{\theta\in\Theta} \left|R(  \theta ) - \overline R( \theta )\right|
&\leq
\sup_{\theta\in\Theta}
{1\over M}\sum_{t=T+1}^{T+M}
\left| \mathbb E_T    L\left(Y_{t},f_{\theta\,t}\right) 
- 
\mathbb E  L\left(Y_{t},f_{\theta\,t}\right) \right|~,
\end{align}
where $\mathbb E_T(\cdot) = \mathbb E(\cdot \vert Y_T,\ldots,Y_1)$.
For each $t=T+1,\ldots,T+M$, it follows from Ibragimov's inequality \citep[Theorem 14.2]{Davidson:1994} that
\begin{align}\label{eqn:ibra}
\mathbb E\left| \mathbb E_T    L(Y_{t}, f_{\theta\,t}) 
- 
\mathbb E  L(Y_{t},f_{\theta\,t}) \right|
&\leq 
6\alpha(l)^{1-1/r}\|L(Y_t,f_{\theta\,t})\|_{L_{r}}~,\quad l = t-T,
\end{align}
for any $r \in \{1,\ldots, r_m\}$.
By combining \eqref{eqn:ibra} with Markov's inequality applied to the non-negative random variable 
$\left| \mathbb E_T    L(Y_{t}, f_{\theta\,t})  
- 
\mathbb E  L(Y_{t},f_{\theta\,t}) \right|$, Proposition \ref{prop:mom_and_mixing}, averaging and taking the supremum on both sides, we have that 
\begin{align}\label{eqn:markov}
\sup_{\theta\in\Theta}	
{1\over M}\sum_{t=T+1}^{T+M}
\left| \mathbb E_T    L(Y_{t}, f_{\theta\,t}) 
- 
\mathbb E  L(Y_{t},f_{\theta\,t}) \right| \leq
{6\sup_{\theta\in\Theta}\|L(Y_t,f_{\theta\,t})\|_{L_{r}}\over \delta M}\sum_{l=1}^M e^{-\left(1-1/r\right)C_\alpha l^{r_\alpha}}
\end{align}
with probability at least $1-\delta$.
Note that Proposition \ref{prop:mom_and_mixing} implies that $\sup_{\theta\in\Theta}\|L(Y_t,f_{\theta\,t})\|_{L_{r}}<\infty$.
It follows from \eqref{eqn:triangular}, \eqref{eqn:ibra} and \eqref{eqn:markov} that if we set $\delta = T^{-1/2}$ 
\begin{align*}
\sup_{\theta\in\Theta} \left|R(\theta ) - \overline R( \theta )\right|
&\geq
{1 \over \sqrt{T} }
{6\sup_{\theta\in\Theta}\|L(Y_t,f_{\theta\,t})\|_{L_{r}}\over  \gamma }\sum_{l=1}^{\infty} e^{-\left(1-1/r\right)C_\alpha l^{r_\alpha}} 
\end{align*}
holds with probability at most $T^{-{1\over 2}} $.
\end{proof}

\begin{proof}[Proof of Proposition \ref{prop:covering}]
	Define \( g_{\theta\,t} = L (Y_t , f_{\theta\,t} ) \).
	Let $\Theta_i = \{ \theta \in \mathbb R^p : \|\theta-\theta_i\|_2 \leq \delta \}$ with $\theta_i \in \Theta$ for $i=1,\ldots,N_\delta$ denote a $\delta$-covering of $\Theta$ for some $\delta\in (0,1]$. 
	Then, we have that
	\begin{eqnarray*}
		&& 
		\mathbb P\left( \sup_{\theta \in \Theta} \left| R_T(\theta) - \overline R(\theta) \right| > {\varepsilon \over 2} \right) 
		\leq \sum_{i=1}^{N_\delta} \mathbb P\left( \sup_{\theta \in \Theta_i} \left| {1\over T} \sum_{t=1}^T g_{\theta\,t} - \mathbb E g_{\theta\,t} \right| > {\varepsilon \over 2 }\right) \\
		&&
		\quad \leq \sum_{i=1}^{N_\delta} \mathbb P\left( \left| {1\over T} \sum_{t=1}^T g_{i\,t}- \mathbb E g_{i\,t} \right| > {\varepsilon\over 4} \right) \\
		&&
		\quad + \sum_{i=1}^{N_\delta} \mathbb P\left( \sup_{\theta \in \Theta_i} \left| {1\over T} \sum_{t=1}^T g_{\theta\,t}- \mathbb Eg_{\theta\,t} 
		- \left( {1\over T} \sum_{t=1}^T g_{i\,t} - \mathbb E g_{i\,t} \right) \right| > {\varepsilon \over 4} \right) ~,
	\end{eqnarray*}
	where $g_{i\,t} = g_{\theta_i\,t}$.
	Lemma \ref{lem:gdiffbound} from the Online Appendix establishes that $(i)$ for each $\theta \in \Theta_i$ we have \( |g_{\theta\,t} -  g_{i\,t}| \leq \bar g_{i\,t} = \delta C_\psi ( d_{\theta_i\,t}^2 + 2 |Y_t-f_{\theta_i\,t}|d_{\theta_i\,t})\)
	and $(ii)$ there exists a positive constant $C_d$ (that does not depend on $i$ and $t$) such that for all $\delta \in (0,1]$ we have that $\mathbb E \bar g_{i\,t} \leq \delta C_d $.
	Set $\delta = \varepsilon / (16 C_d)$. Then, for all $\varepsilon < 16 C_d $ 
	\begin{eqnarray*}
		&&
		\mathbb P\left( \sup_{\theta \in \Theta_i} \left| {1\over T} \sum_{t=1}^T ( g_{\theta\,t} - g_{i\,t} )
		- \mathbb E ( g_{\theta\,t} - g_{i\,t} ) \right| > {\varepsilon \over 4} \right) \leq \mathbb P\left( {1\over T} \sum_{t=1}^T (\bar g_{i\,t} + \mathbb E \bar g_{i\,t}) > { \varepsilon \over 4} \right) \\
		&& \quad = \mathbb P\left( {1\over T} \sum_{t=1}^T (\bar g_{i\,t} - \mathbb E \bar g_{i\,t}) > { \varepsilon \over 4} - 2 \mathbb E \bar g_{i\,t} \right) 
		\leq \mathbb P\left( {1\over T} \sum_{t=1}^T (\bar g_{i\,t}- \mathbb E \bar g_{i\,t}) > {\varepsilon \over 8} \right) 
	\end{eqnarray*}
	holds, where the last inequality follows from the fact that $\mathbb E \bar g_{i\,t} \leq \varepsilon/16$.
	The claim follows after defining $U_{\theta_i\,t} = g_{i\,t}$ and $V_{\theta_i\,t} = C_\psi \left( d_{\theta_i\,t}^2 + 2 |Y_t-f_{\theta_i\,t}| d_{\theta_i\,t}\right) = 16 C_d \bar g_{i\,t} / \varepsilon $
	and noting that $N_\delta \leq \left(1 + 2C_\Theta / \delta \right)^p = \left(1 + 32 C_\Theta C_d  / \varepsilon \right)^p $.
\end{proof}

\begin{proof}[Proof of Proposition \ref{prop:conc}]
	The analysis of the sequences $\{ \widetilde U_{\theta\,t}, t\geq0 \}$ and $\{ \widetilde V_{\theta\,t} , t\geq 0\}$ is analogous. 
	Here we focus on $\{ \widetilde U_{\theta\,t} , t\geq 0\}$.
	To simplify notation we omit the subscript $\theta$ in the notation of the sequence $\{ \widetilde U_{\theta\,t}, t\geq0 \}$.
	
	Let \( \sum_{t=1}^T \widetilde U_{t} = \sum_{t=1}^T U'_{t} + \sum_{t=1}^T U''_{t} \) where \( U'_{t} =  \widetilde{U}_{t} \mathbbm{1}_{\{|\widetilde{U}_{t}| \leq b_T\}} - \mathbb E \left(\widetilde{U}_{t} \mathbbm{1}_{\{|\widetilde{U}_{t}| \leq b_T\}} \right) \) and 
	\( U''_{t}  = \widetilde{U}_{t} \mathbbm{1}_{\{|\widetilde{U}_{t}| > b_T\}}   - \mathbb E \left(\widetilde{U}_{t} \mathbbm{1}_{\{|\widetilde{U}_{t}| > b_T\}} \right) \).
	We then have that 
	\begin{eqnarray*}
		\mathbb P\left( \left| {1\over T} \sum_{t=1}^T \widetilde{U}_{t} \right| > { \varepsilon_T \over 4} \right) 
		\leq \mathbb P\left( \left|\sum_{t=1}^T U'_{t} \right| > {T \varepsilon_T \over 8} \right) + 
		\mathbb P\left( \left|\sum_{t=1}^T  U''_{t} \right| > {T \varepsilon_T \over 8} \right) ~.
	\end{eqnarray*}
	Define $M_T = \lfloor T^{{1\over 2}-{p+1\over 2(r_m-1)}} \log^{-{1\over 2}} T \rfloor$ and $ b_T = C_b T^{p+1 \over 2(r_m-1) } (p \log T)^{-{p-1\over 2(r_m-1)}}$ where $C_b$ is a positive constant to be chosen in what follows.
	The sequence $\{ U'_{t} \}_{t=1}^T$ has the same mixing properties of $\{\widetilde{U}_{t} \}_{t=1}^T$ and $\| U_{t}' \|_{L_\infty} < 2b_T$. 
	Then for all $T$ sufficiently large 
	and $p < r_m - 2$ the conditions of Theorem 2.1 in \citet{Liebscher:1996} are satisfied since 
	$ M_T \in \{1,\dots,T\} $ and \( 8 M_T b_T  \leq {T \varepsilon_T / 2} \).
	Then, we have
	\begin{align*}
		\mathbb P\left( \left| \sum_{t=1}^T U'_t \right| > {T \varepsilon_T\over 2} \right) 
		&\leq 4 \exp\left( - { T^2 \varepsilon_T^2 \over 4096 { T\over M_T} \mathbb E ( \sum_{t=1}^{M_T} U'_{t} )^2  + {64 \over 3} M_T b_T T \varepsilon_T  }  \right)\\
		&\quad+ 4{ T \over M_T } \exp\left(-C_\alpha M_T^{r_\alpha}\right) ~.
	\end{align*}
	If we define $\gamma(l) = | \mbox{Cov}(U'_{t},U'_{t+l})|$ for $l=0,\ldots,T-1$,
	then we have that $\mathbb E ( \sum_{t=1}^{M_T} U'_{t} )^2 \leq M_T ( \gamma(0) + 2 \sum_{l=1}^{\infty} \gamma(l) )$.
	Let $C_m = \sup_{\theta\in\Theta}\|\widetilde{U}_t\|_{L_{r_m}}$.
	Noting that $L(Y_t,f_{\theta\,t})\geq0$, Davydov's inequality implies 
	\begin{eqnarray*}
		\gamma(l) & \leq &
		2 {r_m\over r_m-2} 2^{1-{2\over r_m}} \alpha(l)^{1 - {2\over r_m}} \| \widetilde{U}_{t}' \|_{L_{r_m}} \| \widetilde{U}_{t+l}' \|_{L_{r_m}} 
		\leq  16 C_m^2  {r_m\over r_m-2} \alpha(l)^{1 - {2\over r_m}} ~,
	\end{eqnarray*}
	for $l=0,\ldots,T-1$, and we use the fact that for any $r$ we have $\|\widetilde{U}_t'\|_{L_r} \leq 2 \|\widetilde{U}_t\|_{L_r}$.
	Thus, we have $ \mathbb E ( \sum_{t=1}^{M_T} U'_{t} )^2 \leq M_T 16 C_m^2 {r_m \over r_m -2}( 1 + 2 \sum_{l=1}^\infty \exp\left(-C_\alpha l^{r_\alpha}\right)^{1-{2\over r_m} }) = M_T \sigma^2 $.
	For all $T$ sufficiently large, we have that
	$ \log \left( 1 + {32 C_\Theta C_d \over \varepsilon_T } \right)^p \leq {1 \over 2} p \log T $.
	Then, for all $T$ sufficiently large, it holds that
	\begin{eqnarray*}
		&& \left( 1 + {32 C_\Theta C_d \over \varepsilon_T } \right)^p  \mathbb P\left( \left|  \sum_{t=1}^T U'_t \right| > {T \varepsilon_T\over 2} \right) \\
		&& \quad \leq 4 \left( 1 + {32 C_\Theta C_d \over \varepsilon_T } \right)^p \left[ \exp \left( - { 2116 \sigma^2 p \log T \over 4096 \sigma^2 + {2944 \over 3} \sigma C_b (p \log T)^{-{p-1\over 2(r_m-1)}}  } \right) 
		+ { T \over M_T } e^{-C_\alpha M_T^{r_\alpha}} \right] \\
		&& \quad = o( \log^{-1} T ) ~.
	\end{eqnarray*}	
	Let $C_m = \sup_{\theta \in \Theta} \| \widetilde U_{t} \|_{L_{r_m}}$. We note that for all $T$ sufficiently large, 
	\begin{align*}
		&\left( 1 + {32 C_\Theta C_d \over \varepsilon_T } \right)^p \mathbb P\left( \left| \sum_{t=1}^T U''_{t} \right| > { {T \varepsilon_T} \over 2 }\right)  
		 \stackrel{(a)}{\leq}  \left( 1 + {32 C_\Theta C_d \over \varepsilon_T } \right)^p{8 \over T \varepsilon_T} \mathbb E \left| \sum_{t=1}^T U''_{t} \right| \\ 
		&\quad\leq\left( 1 + {32 C_\Theta C_d \over \varepsilon_T } \right)^p {16 \over \varepsilon_T} \mathbb E \left|\widetilde{U}_{t} \mathbbm{1}_{\{|\widetilde{U}_{t}| > b_T\}} \right| 
		\stackrel{(b)}{\leq}  \left( 1 + {32 C_\Theta C_d \over \varepsilon_T } \right)^p{16 \over \varepsilon_T } { C_m^{r_m} \over b_T^{r_m-1} }  \\
		&\quad\leq  \left( 1 + 32 C_\Theta C_d \right)^p{16 \over \varepsilon_T^{p+1} } { C_m^{r_m} \over b_T^{r_m-1} }  
		\stackrel{(c)}{\leq} \log^{-1} T ~,
	\end{align*}
	where $(a)$ follows from Markov's inequality $(b)$ from the inequality $\mathbb E( | X \mathbbm{1}_{\{| X|>b\}} | ) \leq \mathbb E(|X|^r)/b^{r-1}$ for any random variable $ X$ wih finite $r$-th moment and positive constant $b$, and $(c)$ from a sufficiently large choice of the constant $C_b$.
	The sequence $\tilde V_{\theta\,t}$ can be analysed using the exact same strategy (using the same choice of $M_T$ and $b_T$ used for $\tilde U_T$).
\end{proof}

\section{Irreducibility, Aperiodicity, Drift Criterion}

Before we proceed, we establish upper bounds on $|H_t|$, $|f_{\theta\,t}|$ and $|d_{\theta\,t}|$.
By Assumption \ref{asm:dgp}$(i)$ we have that for any $\epsilon >0$ there exists some $1<M_\epsilon<\infty$ such that
$
|H_t| \leq 
(a_h+b_h|Z_{1\,t}|+\epsilon)|H_{t-1}|$
for all $ |H_{t-1}|>M_\epsilon$.
The same assumption also implies that when $|H_{t-1}|\leq M_\epsilon$ we have
$
|H_t|\leq |g_{h1}(H_{t-1})| + |g_{h2}(H_{t-1})|~|Z_{1\,t}| \leq \overline g_h^\epsilon( 1 + |Z_{1\,t}|),
$
where $\overline g_h^\epsilon = \sup |g_{h1}(h)| \vee \sup|g_{h2}(h)|<\infty$ and the supremums are taken with respect to $h \in [-M_\epsilon,M_\epsilon]$.
Hence we have that
\begin{align}\label{eqn:hbounds}
|H_t| \leq 
(a_h+b_h|Z_{1\,t}|+\epsilon)|H_{t-1}|\mathbbm 1_{ \{|H_{t-1}|>M_\epsilon\}}
+\overline g_h^\epsilon (1 + |Z_{1\,t}|)\mathbbm 1_{\{ |H_{t-1}|\leq M_\epsilon\}}~.
\end{align}
By \ref{asm:algo} we have $|f_{\theta\,t}|\leq \overline \alpha_0 + \overline \alpha_1 |g_{y1}(H_{t-1})|+\overline \alpha_1|g_{y2}(H_{t-1})|~|Z_{2\,t}| + \overline \beta_1 |f_{\theta\,t-1}|$. Furthermore, it follows from Assumption \ref{asm:dgp}$(i)$ and $(ii)$ that
\begin{align}\label{eqn:fbounds}
|f_{\theta\,t}| 
&\leq
\begin{cases}
\overline \alpha_1 C_{y} (1 +\epsilon+ |Z_{2\,t}|)~|H_{t-1}| + \overline \beta_1 |f_{\theta\,t-1}| & |H_{t-1}| > M_\epsilon\\
\overline \alpha_0 + \overline \alpha_1 \overline g_y^\epsilon(1 + |Z_{2\,t}|)+ \overline \beta_1 |f_{\theta\,t-1}|& |H_{t-1}|\leq M_\epsilon
\end{cases}~,
\end{align}
where $C_y = C_{y1}\vee C_{y2}$, $\overline g_y^\epsilon = \sup |g_{y1}(h)|\vee \sup |g_{y2}(h)|<\infty$ and the supremums are taken with respect to $h \in [-M_\epsilon,M_\epsilon]$. We note that $M_\epsilon$ can be redefined if necessary in order to remove the constant $\overline \alpha_0$ from the bound when $|H_{t-1}|>M_\epsilon$.
Lastly, using analogous arguments we have 
\begin{align}\label{eqn:dbounds}
|d_{\theta\,t}| &\leq 
\begin{cases}
C_{y}(1+\epsilon+|Z_{2\,t}|)~|H_{t-1}|+|f_{\theta\,t-1}| + \overline\beta_1 |d_{\theta\,t-1}|~,
& |H_{t-1}|>M_\epsilon\\
1 + \overline g_y^\epsilon(1+|Z_{2\,t}|) + |f_{\theta\,t-1}| + \overline\beta_1 |d_{\theta\,t-1}| ~,
& |H_{t-1}|\leq M_\epsilon
\end{cases}
\end{align}
where again $M_\epsilon$ may be redefined if necessary in order to remove the constant $1$ from the bound  when $|H_{t-1}|>M_\epsilon$.\\
Second, we introduce a partition of the state space $\mathcal X$ that plays a key role in the subsequent proofs.
Let $\kappa = (\kappa_h,\kappa_f,\kappa_d)'\in(0,1)^3$ where the specific choice of this vector will be determined in what follows.
We define the sets
\begin{align}\label{eqn:S1_S2}
S_{2\,\epsilon} = \{ (h,f,d)\in \mathcal X: \kappa_h|h|+\kappa_f|f|+\kappa_d|d| \leq M \} \quad \text{and} \quad S_{1\,\epsilon} = \mathcal X \setminus S_{2\,\epsilon},
\end{align}
where $M$ is a positive constant (note that in general $M \neq M_\epsilon$). \\
Third, let $\rho_{z\,\epsilon} = a_h + b_h|Z_{1\,t}| + \epsilon$, $C_{y,z}^\epsilon = C_y(1 + \epsilon + |Z_{2\,t}|)$ and define the matrix \[
	\mathbf C_\epsilon(Z_t) = 
	\begin{bmatrix}
	\rho_{z\,\epsilon} & 0 & 0 \\
	\overline \alpha_1 \tilde C_{y,z}^\epsilon & \overline \beta_1+\epsilon & 0\\
	\tilde C_{y,z}^\epsilon &  1 &  \overline \beta_1+\epsilon 
	\end{bmatrix} ~.
\]
Assumption \ref{asm:dgp}$(iii)$ implies that $\mathbb E \left( \mathbf C_{\epsilon}(Z_t)^{\otimes 2r_m} \right)$ exists.\\
Finally, we set $\epsilon>0$ in a way such that $\mathbb E( a_h+b_h|Z_{1\,t}|+\epsilon)^{2r_m}<1$ and $\overline \beta_1 + \epsilon < 1$. 
Assumptions \ref{asm:dgp}$(iv)$ and \ref{asm:algo} imply that an $\epsilon$ that satisfies these constraints exists.
For this particular choice of $\epsilon$, we have that the spectral radius of $\mathbb E\left(\mathbf C_\epsilon(Z_t)^{\otimes 2r_m}\right)$ is strictly less than unity. Such a choice of $\epsilon$ will be assumed throughtout.

\begin{lemma}[Irreducibility]\label{lem:irre}
	Consider the setting of Proposition \ref{lem:geometric_erg:x}.
	There exists an open rectangular region $D \subset \mathcal X$ that does not depend on $\theta$ or $x$ such that the Markov chain $\{X_{\theta\,t},t\geq0\}$ is $\varphi$-irreducible with $\varphi(A)= \mu_{Leb}(A \cap D)$ for any $A \in \mathcal B(\mathcal X)$.
\end{lemma}

\begin{proof}
	We follow the strategy of \citet{Lanne:Saikkonen:2005} and \citet{Meitz:Saikkonen:2008}.
	It suffices to show the following three intermediate results. 
	\begin{enumerate}[I.]
		\item For any $x \in S_{1\,\epsilon}$ there exists an $n \in \mathbb Z_+$ such that $ P_X^n (x, S_{2\,\epsilon})>0$.
		\item For any $A \in \mathcal B(\mathcal X)$ it holds that $\inf_{x \in S_{2\,\epsilon}} P_X^2(x,A \cap D) \geq c_* \mu_{Leb}(A \cap D)$,
		where $D$ is an open rectangular region to be specified in what follows
		and $c_*$ is a positive scalar that does not depend on $\theta$ or $x$.
		\item For any $x \in S_{1\,\epsilon}$ there exists an $n \in \mathbb Z_+$ such that for any $A \in \mathcal B(\mathcal X)$ it holds that $P_X^{n + 2}(x, A \cap D) >0$ whenever $\mu_{Leb}(A \cap D) > 0$.	
	\end{enumerate}
	
	\noindent {I}. 
	Define the event 
	$\Omega_n = \left\lbrace \omega \in \Omega: |Z_{1\,t}|\leq \mathbb E|Z_{1\,t}| \text{ and } |Z_{2\,t}|\leq \mathbb E|Z_{2\,t}|~,t = 1,\dots,n \right\rbrace$ for an arbitrary $n$ 
	and note that $\mathbb P(\Omega_n)>0$.
	Define the auxiliary vector $\dot X_{\theta\,t} = (|H_t|, |f_{\theta\,t}| , |d_{\theta\,t}| )'$.
	To establish part I we show that for any $\kappa \in (0,1)^3$ and for each $t=1,\ldots,n$ we have that when $X_{\theta\,t-1} \in S_{1\,\epsilon}$ the inequality 
	\begin{align}\label{eqn:quadform}
	\left(\kappa'\dot X_{\theta\,t}\right)^{2r_m}
	\leq \left(\kappa^{\otimes 2r_m}\right)'\mathbb E\left(\mathbf C_\epsilon(Z_t)^{\otimes 2r_m}\right)\dot X_{\theta\,t-1}^{\otimes 2r_m} ~
	\end{align}
	holds given $\Omega_n$.
	We distinguish the cases $(i)$ $|H_{t-1}| > M_\epsilon$ and $(ii)$ $|H_{t-1}|\leq M_\epsilon$.\\
	$(i)$ From \eqref{eqn:hbounds}, \eqref{eqn:fbounds} and \eqref{eqn:dbounds} we have that
	\begin{align}\label{eqn:kronx}
	\left(\kappa'\dot X_{\theta\,t}\right)^{2r_m}
	&\leq \left( \kappa' \mathbf C_\epsilon(Z_t) \dot X_{\theta\,t-1}\right)^{2r_m}
	= \left(\kappa^{\otimes 2r_m}\right)'\mathbf C_\epsilon(Z_t)^{\otimes 2r_m}\dot X_{\theta\,t-1}^{\otimes 2r_m},  
	\end{align}
	where the last equality follows from properties of Kronecker products.
	By adding and subtracting $\mathbb E \left(\mathbf C_\epsilon(Z_t)^{\otimes 2r_m}\right)$ in \eqref{eqn:kronx} we obtain
	\begin{align}\label{eqn:addsubtract}
	\left(\kappa'\dot X_{\theta\,t}\right)^{2r_m}
	&\leq  \left(\kappa^{\otimes 2r_m}\right)'\mathbb E\left(\mathbf C_\epsilon(Z_t)^{\otimes 2r_m}\right)\dot X_{\theta\,t-1}^{\otimes 2r_m} \nonumber\\
	&+ 
	\left(\kappa^{\otimes 2r_m}\right)'\left\lbrace \mathbf C_\epsilon(Z_t)^{\otimes 2r_m}- \mathbb E\left(\mathbf C_\epsilon(Z_t)^{\otimes 2r_m}\right)\right\rbrace\dot X_{\theta\,t-1}^{\otimes 2r_m}~.
	\end{align}
	The random elements of the matrix 
	$\mathbf C_\epsilon(Z_t)^{\otimes 2r_m}$ are of the form
	$C(a_h+b_h|Z_{1\,t}|+\epsilon)^{j_1}(1 + \epsilon + |Z_{2\,t}|)^{j_2}$ for $j_1,j_2\geq 0$ such that $j_1+j_2\leq 2r_m$, where $C$ denotes some positive constant (that depends on $\epsilon$). 
	Conditionally on $\Omega_n$, it follows from the independence between $Z_{1\,t}$ and $Z_{2\,t}$ as well as Jensen's inequality that the random elements are bounded from above by their expectations.
	This establishes that the bound in \eqref{eqn:quadform} holds in case $(i)$.\\
	$(ii)$ From \eqref{eqn:hbounds}, \eqref{eqn:fbounds} and \eqref{eqn:dbounds} we can write
	$
	\left(\kappa'\dot X_{\theta\,t}\right)^{2r_m} 
	\leq \left( \overline C_{z\,\epsilon}  + 
	\kappa_{\mathbf f}'\overline B \mathbf{f}_{\theta\,t-1} \right)^{2r_m},
	$
	where
	$\overline C_{z\,\epsilon} = {\kappa_{h} \overline H_{z\,\epsilon} + \kappa_{f} (\overline \alpha_0 + \overline \alpha_1\overline Y_{z\,\epsilon}) + \kappa_{d}(1 + \overline Y_{z\,\epsilon})}$,
	$\overline H_{z\,\epsilon} = \overline g_h^\epsilon (1 + |Z_{1\,t}|)$,
	$\overline Y_{z\,\epsilon} = \overline g_y^\epsilon(1 +|Z_{2\,t}|)$,
	$\kappa_{\mathbf f}=(\kappa_{f},\kappa_{d})'$,
	$\mathbf{f}_{\theta\,t-1} = \left(|f_{\theta\,t-1}|,|d_{\theta\,t-1}|\right)'$ and 
	$\overline B$ is a $2\times2$ lower triangular matrix with $\overline B_{11}=\overline B_{22}=\overline \beta_1$ and $\overline B_{21}=1$.
	On the event $\Omega_n$, it is straightforward to verify that 
	$\overline C_{z\,\epsilon} \leq 
	\left\|\overline C_{z\,\epsilon}\right\|_{L_1}$.
	Combining conditions $X_{\theta\,t-1} \in S_{1\,\epsilon}$ and $|H_{t-1}| \leq M_\epsilon$ we have
	$ \kappa_{\mathbf f}'\mathbf{f}_{\theta\,t-1} > M- \kappa_{h}|H_{t-1}| \geq M-\kappa_{h} M_\epsilon $.
	Thus, we can choose $M$ large enough such that 
	$\epsilon \left(M - \kappa_{h} M_\epsilon\right) >  
	\left\|\overline C_{z\,\epsilon}\right\|_{L_{2r_m}}\geq
	\left\|\overline C_{z\,\epsilon}\right\|_{L_1}$. Note that this choice of $M$ is independent of $t$.
	Such a choice of $M$ is kept fixed throughtout our derivations.\footnote{%
	We remark that Assumption \ref{asm:dgp} implies that $\left\|\overline C_{z\,\epsilon}\right\|_{L_{2r_m}}$ exists.}
	Then,  conditionally on $\Omega_n$ and whenever $ X_{\theta\,t-1}\in S_{1\,\epsilon}$, we have
	\begin{align*}
	\left(\overline C_{z\,\epsilon} + \kappa_{\mathbf f}'\overline B \mathbf{f}_{\theta\,t-1} \right)^{2r_m}
	\leq 
	\left(\left\|\overline C_{z\,\epsilon}\right\|_{L_{2r_m}} + \kappa_{\mathbf f}'\overline B \mathbf{f}_{\theta\,t-1} \right)^{2r_m}
	\leq 
	\left(\kappa_{\mathbf f}'\overline B_\epsilon \mathbf{f}_{\theta\,t-1} \right)^{2r_m}~,
	\end{align*}
	where $\overline B_\epsilon = \overline B + \epsilon I$.
	We note that
	$\kappa_{\mathbf f}'\overline B_\epsilon \mathbf{f}_{\theta\,t-1} = \kappa' \mathbf C_\epsilon(Z_t) \ddot X_{\theta\,t-1}$, where $\ddot X_{\theta\,t-1} = \left( 0,\mathbf{f}_{\theta\,t-1}'\right)'$, and we have that
	\begin{align}\label{eqn:ddot}
	\left(\kappa' \mathbf C_\epsilon(Z_t) \ddot X_{\theta\,t-1} \right)^{2r_m}
	\leq 
	(\kappa^{\otimes 2r_m})'\mathbb E\left(\mathbf C_\epsilon(Z_t)^{\otimes 2r_m}\right)\dot X_{\theta\,t-1}^{\otimes 2r_m} ~ ,
	\end{align}
	where we use the definition of $\ddot X_{\theta\,t-1}$.
	This establishes that the bound in \eqref{eqn:quadform} holds in case $(ii)$.	
	By Lemma A.2. of \cite{Ling:McAleer:2003} we can choose $\kappa \in (0,1)^3$ such that the vector $v = \left(I - \mathbb E\left(\mathbf C_\epsilon(Z_t)^{\otimes 2r_m}\right)\right)'\kappa^{\otimes 2r_m}$ has positive components.\footnote{Recall that the matrix $\mathbb E\left(\mathbf C_\epsilon(Z_t)^{\otimes 2r_m}\right)$ has a spectral radius that is strictly less than unity. As noted by \cite{Lanne:Saikkonen:2005}, the given proof makes clear that it means no loss of generality to assume that the components of $\kappa$ are bounded by unity.} In particular, we remark that the vector $v$ does not depend on $\theta$. We use $\underline v$ to denote the minimum of the components of $v$.
	Thus from \eqref{eqn:quadform} it follows that
	\begin{align}\label{eqn:contract_1}
	\left(\kappa'\dot X_{\theta\,t}\right)^{2r_m} 
	&\leq 
	\left(\kappa'\dot X_{\theta\,t-1}\right)^{2r_m}-v'\dot X_{t-1}^{\otimes 2r_m}
	=\left(\kappa'\dot X_{\theta\,t-1}\right)^{2r_m}\left(1 - {v'\dot X_{\theta\,t-1}^{\otimes 2r_m}\over \left(\kappa^{\otimes 2r_m}\right)'\dot X_{\theta\,t-1}^{\otimes 2r_m} } \right)\nonumber\\
	&\leq
	(1 - \underline v)\left(\kappa'\dot X_{\theta\,t-1}\right)^{2r_m}~,
	\end{align}
	where $\underline v \in (0,1)$. By repeated application of \eqref{eqn:contract_1} starting from $X_{\theta\,0}=x\in S_{1\,\epsilon}$ we have
	$
	\left(\kappa'\dot X_{\theta\,n}\right)^{2r_m} \leq (1-\underline v)^n \left(\kappa'\dot x\right)^{2r_m}.
	$
	Since $\underline v \in (0,1)$ we have that for any $x \in S_{1\epsilon}$ there exists a sufficiently large $n$ such that the right hand side of the inequality is smaller than $M^{2r_m}$.
	Thus, we have that $X_{\theta\,n} \in S_{2\,\epsilon}$ with positive probability.\\
	\noindent {II.} 
	First we write
	$
	P_X^2(x,A \cap D) 
	=  \mathbb E ( \; \mathbb E( \mathbbm{1}_{\{X_{\theta\,2} \in A \cap D\}} | H_1 , X_{\theta\,0} ) \; | \; X_{\theta\,0} = x \; ),
	$
	for any $x \in S_{2\,\epsilon}$, $A \in \mathcal B(\mathcal X)$, $D \subset \mathcal X$ such that $D$ is an open rectangular region (to be specified in what follows).
	Let $\underline h_1 \geq \sup_{|h| < M /\kappa_h }g_{h1}(h)$.
	The result is obtained by showing the following intermediate results.
	$(i)$ $\inf_{h_1 \in [\underline h_1,\underline h_1+1]} \mathbb E( \mathbbm{1}_{\{X_{\theta\,2} \in A \cap D\}} | H_1=h_1 , X_{\theta\,0}=x ) \geq c' \mu_{Leb}(A \cap D),$ where $c'$ is a positive scalar that does not depend on $\theta$ or $x$.
	$(ii)$ $P_X^2(x,A \cap D) \geq  c'' \inf_{h_1 \in [\underline h_1,\underline h_1+1]} 
		\mathbb E( \mathbbm{1}_{\{X_{\theta\,2} \in A \cap D\}} | H_1=h_1 , X_{\theta\,0}=x )$,
		where $c''$ is a positive scalar that does not depend on $\theta$ or $x$. \\ 
	$(i)$ Set 
	$ \underline Z_{2\,1} = \sup_{|h| < M/\kappa_h} {(R - g_{y1}(h) )/ g_{y2}(h)}$ and $\underline Z_{2\,2} = \sup_{h_1 \in [\underline h_1,\underline h_1+1]} {(R - g_{y1}(h_1) ) / g_{y2}(h_1)}$,
	where
	$
	R >  { \overline \beta_{1} M / \kappa_f - \underline \alpha_{0}  \over \underline \alpha_{1}} \vee 
	r_K \vee 
	\sup_{|h| < M/\kappa_h } g_{y1}(h) \vee 
	\sup_{h_1 \in [\underline h_1,\underline h_1+1]} g_{y1}(h_1),
	$
	and note that  $\underline Z_{2\,1}$ and $\underline Z_{2\,2}$ do not depend on $x$, $h_1$ or $\theta$.
	Then it holds that
	\begin{align*}
	& \inf_{h_1 \in [\underline h_1,\underline h_1+1]} \mathbb E( \mathbbm{1}_{\{X_{\theta\,2} \in A \cap D\}} | H_1=h_1 , X_{\theta\,0}=x ) \nonumber \\
	& \quad \geq \inf_{h_1 \in [\underline h_1,\underline h_1+1]} 
	\int_0^\infty
	\int_{\underline Z_{2\,1}}^\infty
	\int_{\underline Z_{2\,2}}^\infty
	\mathbbm{1}_{\{X_{\theta\,2} \in A \cap D\}}
	\phi_H(Z_{1\,2}) 
	\phi_Y(Z_{2\,1})
	\phi_Y(Z_{2\,2})
	dZ_{1\,2}
	dZ_{2\,1}
	dZ_{2\,2} ~. \label{eqn:part2} 
	\end{align*}
	Over the integration range of the right hand side a number of properties hold.
	First, we have that $g_{y1}(h)+g_{y2}(h)Z_{2\,1}>R$ and $g_{y1}(h_1)+g_{y2}(h_1)Z_{2\,2} > R$, and $f_{\theta\,1} > 0$. 
	Furthermore, we have that the map between $X_{\theta\,2}$ and $(Z_{1\,2}, Z_{2\,2}, Z_{2\,1})'$ is linear and is given by 
	\begin{equation}\label{eqn:XZmap}
	X_{\theta\,2} = c + G 
	\begin{bmatrix}
	Z_{1\,2} & Z_{2\,2} & Z_{2\,1} 
	\end{bmatrix}' ~,
	\end{equation}
	where the expression for the vector $c=(c_h,c_f,c_d)'$ and the 3 $\times$ 3 block-diagonal matrix $G$ are given in \eqref{eqn:c_vector} and \eqref{eqn:G_matrix} from the Online Appendix.
	Furthermore, $G$ is invertible uniformly over $S_{2\,\epsilon}$, $\Theta $ and $[\underline h_1,\underline h_1+1]$. 
	In fact, we have that 
	$
	\det G \in [G_l, G_u]$,
	and it is clear that $G_l > 0$ by asssumptions \ref{asm:dgp} and \ref{asm:algo}.\footnote{%
		See \eqref{eqn:G_l} and \eqref{eqn:G_u} from the Online Appendix for the expression for $G_l$ and $G_u$.
	}
Define $Z_{1\,2}(X_{\theta\,2}) = {H_2 - c_h \over g_{h2}(h_1)},$ and
	\begin{eqnarray*}
		Z_{2\,2}(X_{\theta\,2}) &=&
		{ g_{h2}(h_1)g_{y2} (h) [ \alpha_{1\,K}+\overline \beta_1)(f_{\theta\,2}-c_f) - \alpha_{1\,K}\beta_{1\,K}(d_{\theta\,2} - c_d) ] \over \det G} \\
		Z_{2\,1}(X_{\theta\,2}) &=&
		{ g_{h2}(h_1)g_{y2} (h_1)[ \alpha_{1\,K}(d_{\theta\,2}-c_d) - (f_{\theta\,2}-c_f) ] \over \det G} ~.
	\end{eqnarray*}
	We observe that the constraints 
	$Z_{2\,1}(X_{\theta\,2})>\underline Z_{1\,2}$ and
	$Z_{2\,2}(X_{\theta\,2})>\underline Z_{2\,2}$
	impose upper and lower bounds on $d_{\theta\,2}$ which are linear functions of $f_{\theta\,2}$ with positive slopes.
	In fact, from Assumption \ref{asm:algo} we have that the minimum discrepancy between slopes is given by
	$
	\inf_{\theta\in\Theta} 
	{\alpha_{1\,K} + \overline \beta_1 \over \alpha_{1\,K}\beta_{1\,K}} 
	- {1\over \alpha_{1\,K}} = 
	\inf_{\theta\in\Theta} 
	{\alpha_{1\,K} + \overline \beta_1 - \beta_{1\,K} \over \alpha_{1\,K}\beta_{1\,K}}=
	{1 \over \overline \beta_1} > 0.
	$
	It follows that the intersection of images of the map defined in \eqref{eqn:XZmap} with respect to $\theta\in\Theta$, $x\in S_{2\,\epsilon}$ and $h_1\in[\underline h_1,\underline h_1+1]$ contains the following set\footnote{%
		See \eqref{eqn:H2_lb}, \eqref{eqn:f2_lb},\eqref{eqn:d2_lb} and \eqref{eqn:d2_ub} from the Online Appendix for the expressions for $\underline H_2$, $\underline f_2$, $\underline d_2(f_{\theta\,2})$ and $\overline d_2(f_{\theta\,2})$.}
	\begin{equation}\label{eqn:intersect_images}
	\{ X_{\theta\,2} \in \mathcal X : H_2 > \underline{H_2}, f_{\theta\,2} > \underline f_{2}, \underline d_2(f_{\theta\,2}) < d_{\theta\,2} < \overline d_2(f_{\theta\,2}) \}~.
	\end{equation}
	We remark that Assumption \ref{asm:algo} implies that such a set is non-empty and it contains sets of positive Lebesgue measure.
	Thus, we can pick $D$ as an open rectangular region in the intersection of \eqref{eqn:intersect_images} and $S_{1\,\epsilon}$. Clearly, $D$ does not depend on $\theta$, $x$ or $h_1$.
	Next, by the change of variable theorem we obtain that
	\begin{align*}
	& \inf_{h_1 \in [\underline h_1,\underline h_1+1]} 
	\int_0^\infty
	\int_{\underline Z_{2\,1}}^\infty
	\int_{\underline Z_{2\,2}}^\infty
	\mathbbm{1}_{\{X_{\theta\,2} \in A \cap D\}}
	\phi_H(Z_{1\,2}) 
	\phi_Y(Z_{2\,1})
	\phi_Y(Z_{2\,2})
	dZ_{1\,2}
	dZ_{2\,1}
	dZ_{2\,2} \\
	&\quad \geq
	\inf_{ \substack{ x \in S_{2\,\epsilon} \\ h_1 \in [\underline h_1,\underline h_1+1] \\\theta \in \Theta \\ X_{\theta\,2} \in A \cap D} }
	\det G^{-1}
	\phi_H(Z_{1\,2}(X_{\theta\,2}))\phi_Y(Z_{2\,1}(X_{\theta\,2}))\phi_Y(Z_{2\,2}(X_{\theta\,2})) 
	\int_{A \cap D}
	dX_{\theta\,2}~.
	\end{align*}
	The boundedness conditions on $g_{h2},g_{y1},g_{y2},\phi_H$ and $\phi_Y$ imply that
	\[
	\inf_{ \substack{ x \in S_{2\,\epsilon} \\ h_1 \in [\underline h_1,\underline h_1+1] \\\theta \in \Theta \\ X_{\theta\,2} \in A \cap D} }
	\det G^{-1}
	\phi_H(Z_{1\,2}(X_{\theta\,2}))\phi_Y(Z_{2\,1}(X_{\theta\,2}))\phi_Y(Z_{2\,2}(X_{\theta\,2})) \geq c' >0 ~,
	\]
	where $c'$ does not depend on $\theta$, $x$ or $h_1$.
	The claim of part $(i)$ then follows.
	\\
	\noindent $(ii)$
	We have that
	\begin{align*}
	P_X^2(x,A \cap D) 
	& \geq \int_0^{\infty}\mathbb E( \mathbbm{1}_{\{X_{\theta\,2} \in A \cap D\}} | H_1=h_1 , X_{\theta\,0}=x )\phi_H(Z_{1\,1})dZ_{1\,1} \\
	&\geq \inf_{h_1 \in [\underline h_1,\underline h_1+1]} \mathbb E( \mathbbm{1}_{\{X_{\theta\,2} \in A \cap D\}} | H_1=h_1 , X_{\theta\,0}=x )\\ &\quad\times \int_{\underline h_1}^{\underline h_1 + 1} {1 \over g_{h2}(h)}\phi_H\left({h_1 - g_{h1}(h) \over g_{h2}(h)}\right)dh_1~,
	\end{align*}
	where the last inequality follows by the choice of $\underline h_1$, and $g_{h2}(h)$ is strictly positive by assumption.
	Moreover, the boundedness conditions on $g_{h1}$, $g_{h2}$, and $\phi_H$ imply that
	$
	\inf_{|h| < M / \kappa_h} g_{h2}(h)^{-1}\phi_H\left( g_{h2}(h)^{-1}(h_1 - g_{h1}(h) ) \right) \geq c'>0,
	$
	where we emphasize that $c'$ does not depend on $x$ or $\theta$. 
	This concludes the second part.
	Combining parts $(i)$ and $(ii)$, the result in II holds with $c_* = c' c'' > 0$.
	
	\noindent {III.} The Chapman-Kolmogorov equations imply that for any $x \in S_{1\,\epsilon}$ it holds that
	\begin{eqnarray*}
		P_X^{n+2}(x,A \cap D) 
		\geq \int_{S_{2\,\epsilon}} P_X^n(x,dx_n)P_X^2(x_n,A \cap D)
		\geq c_* \mu_{Leb}(A \cap D) P_X^n(x,S_{2\,\epsilon}) > 0,
	\end{eqnarray*}
	where the last two inequalities follow, respectively, from Parts II and I (for a sufficiently large $n$).
\end{proof}

\begin{lemma}[Aperiodicity]\label{lem:aperiod}
	Consider the setting of Lemma \ref{lem:geometric_erg:x}. Then, the Markov chain $\{X_{\theta\,t},t\geq0\}$ is aperiodic.
\end{lemma}
\begin{proof}
	It follows from Proposition A1.1 in \cite{Chan:1990}, that to establish aperiodicity it suffices to show that that for each $x \in D$  
	there exists an $n \in \mathbb Z_+$ such that 
	$P_X^{n+2}(x,D) >0$ and $P_X^{n+3}(x,D) > 0$, where $D$ is a small set.
	We divide the proof in three parts. 
	In part $(i)$ we show that the set $D$ defined in Lemma \ref{lem:irre} is a small set.
	In part $(ii)$ we show that for each $x \in D$ there exists an $n$ such that $P_X^{n+2}(x,D) >0$.
	In part $(iii)$ we show that for the same $x$ and same $n$ defined in part $(ii)$ it holds that $P_X^{n+3}(x,D) >0$. \\
	\noindent $(i)$ We note that by repeating the arguments in Part II of Lemma \ref{lem:irre} with $S_{2\,\epsilon}$ replaced by $D$ we have that for any $A \in \mathcal B(\mathcal X)$ there exist $c_*'>0$ and an open rectangular region $D'$ such that
	$
	\inf_{x \in D} P_X^2(x,A \cap D') \geq c_*'\mu_{Leb}(A\cap D').
	$ \\
	\noindent $(ii)$ 
	It follows from Parts I and II of Lemma \ref{lem:irre} that
	for any $x \in D$ there exists an $n$ such that \( P_X^{n}( x , S_{2\,\epsilon} ) > 0 \) and 
	for any $x \in S_{2\,\epsilon}$ we have that \( P_X^2( x , D ) > 0 \). 
	The claim follows from an application of the Chapman-Kolmogorov equation.\\
	\noindent $(iii)$ Note that in the proof of Lemma \ref{lem:irre} we can choose an $M$ and $M'$ with $M>M'$ such that $P_X^{n}(x,S_{2\,\epsilon}) > 0$ 
	and $P_X^{n+1}(x,S_{2\,\epsilon}') > 0$ where $S_{2\,\epsilon}' = \{(h,f,d)' \in \mathcal X: \kappa_h |h| + \kappa_f|f| + \kappa_d |d| \leq M' \}$.
	It is straightforward to see in the proof of Lemma \ref{lem:irre} that $M$ can be chosen as any sufficiently large constant.
	Furthermore, we have that $\inf_{x\in S_{2\,\epsilon}'} P_X^2(x,D) \geq \inf_{x\in S_{2\,\epsilon}} P_X^2(x,D) \geq c_* \mu_{Leb}(D) > 0$. 
	The Chapman-Kolmogorov equation implies the claim since we have that
	$
	P_X^{n+3}(x,D) \geq \int_{S_{2\,\epsilon}'} P_X^{n+1}(x,dx_{n+1})P_X^2(x_{n+1},D) \geq c_*\mu_{Leb}(D) P^{n+1}_X(x,S_{2\,\epsilon}') > 0.$
\end{proof}

\begin{lemma}[Drift Criterion]\label{lem:drift}
	Consider the setting of Lemma \ref{lem:geometric_erg:x}. 
	Then, the Markov chain $\{X_{\theta\,t},t\geq 0\}$ satisfies $\mathbb E\left(q_X(X_{\theta\,t})\vert X_{\theta\,t-1}=x \right) \leq (1-\gamma_1) q_X(x) + \gamma_2\mathbbm 1_{\{x \in C\}}$
	for some $\gamma_1>0$ and $\gamma_2 < \infty$ where $C$ is a small set. Furthermore, $\gamma_1$, $\gamma_2$ and $C$ do not depend on $\theta$.
\end{lemma}
\begin{proof}	
	Set $C$ equal to $S_{2\,\epsilon}$ and note that Part II in the proof of Lemma \ref{lem:irre} establishes that $S_{2\,\epsilon}$ is a small set that does not depend on $\theta$.\\
	When $x \in S_{1\,\epsilon}$, we distinguish two cases: $(i)$ $|h|>M_\epsilon$ or $(ii)$ $|h|\leq M_\epsilon$. \\
	Case $(i)$.
	From \eqref{eqn:kronx} we have that
	\begin{align}\label{eqn:Ekronx}
	\mathbb E_x (q_X(X_{\theta\,t}))-1 
	&\leq\mathbb E_x \left( \kappa' \mathbf C_\epsilon(Z_t) \dot x\right)^{2r_m}
	= \left(\kappa^{\otimes 2r_m}\right)' \mathbb E \left(\mathbf C_\epsilon(Z_t)^{\otimes 2r_m}\right)\dot x^{\otimes 2r_m} ~.
	\end{align}
	Following steps analogous to the ones used to go from \eqref{eqn:quadform} to \eqref{eqn:contract_1} we have that
	\begin{align*}
	\mathbb E_x\left(q_X(X_{\theta\,t})|X_{\theta\,t-1}=x\right)
	&\leq 1+(\kappa'\dot x)^{2r_m}-v'\dot x^{\otimes 2r_m}
	\leq (1 - \gamma_1)q_X(x) ~,
	\end{align*}
	where $\gamma_1 \in (0,1)$ and does not depend $\theta$. \\
	Case $(ii)$. From Part I of Lemma \ref{lem:irre} (case $(ii)$) it follows that
	$
	\mathbb E_x \left(q_X(X_{\theta\,t})\right)-1 
	\leq \mathbb E_x \left(\overline C_{z\,\epsilon} + \kappa_{\mathbf f}'\overline B \mathbf f\right)^{2r_m}$.
	We observe that
	\begin{align*}
	\mathbb E_x\left(\overline C_{z\,\epsilon} + \kappa_{\mathbf f}'\overline B \mathbf{f} \right)^{2r_m} =
	\left(\left(\mathbb E_x\left(\overline C_{z\,\epsilon} + \kappa_{\mathbf f}'\overline B \mathbf{f}\right)^{2r_m} \right)^{1 \over 2r_m}\right)^{2r_m}
	&\leq \left( \left\|\overline C_{z\,\epsilon}\right\|_{L_{2r_m}} + \kappa_{\mathbf f}'\overline B \mathbf{f} \right)^{2r_m}~,
	\end{align*}
	and note that the assumptions on the innovations imply that $\left\|\overline C_{z\,\epsilon}\right\|_{L_{2r_m}}$ exists. 
	Using steps analogous to those used to get to \eqref{eqn:ddot} we have that
	\begin{align*}
	\left( \left\|\overline C_{z\,\epsilon}\right\|_{L_{2r_m}} + \kappa_{\mathbf f}'\overline B \mathbf{f} \right)^{2r_m}
	&\leq 
	\mathbb E_x\left(\kappa' \mathbf C_\epsilon(Z_t) \ddot x\right)^{2r_m}
	\leq 
	(\kappa^{\otimes 2r_m})'\mathbb E\left(\mathbf C_\epsilon(Z_t)^{\otimes 2r_m}\right)\dot x^{\otimes 2r_m} ~.
	\end{align*}
	The claim of case $(ii)$ then follows using the same steps of case $(i)$ after equation \eqref{eqn:Ekronx}.
	When $x \in S_{2\,\epsilon}$ it follows from Assumption \ref{asm:dgp} and the definition of $S_{2\,\epsilon}$ that 
$
		\sup_{ \substack{ x \in S_{2\,\epsilon} \\ \theta \in \Theta} }
		\mathbb E\left(q_X(X_{\theta\,t})\vert X_{\theta\,t-1}=x \right) 
		\leq
		\gamma_2 < \infty,
$
	where we have used the fact the expectation exists and it is bounded over $\Theta$ for every $x \in S_{2\,\epsilon}$ provided
	that $Z_{1\,t}$ and $Z_{2\,t}$ have $2 r_m$ moments. Since $(1-\gamma_1)q_X(x)$ is positive the claim of the proposition holds when $x \in S_{2\,\epsilon}$.
\end{proof}

\bibliography{references}
\bibliographystyle{natbib}


\end{document}